\documentclass[11pt]{article}

\usepackage[final]{acl}
\usepackage{multirow}
\usepackage{times}
\usepackage{latexsym}
\usepackage{amsmath}
\usepackage{svg}
\usepackage{colortbl}
\usepackage{xcolor}
\usepackage[T1]{fontenc}
\usepackage{rotating}
\usepackage[utf8]{inputenc}

\usepackage{microtype}

\usepackage{inconsolata}
\usepackage{booktabs}

\usepackage{graphicx}
\usepackage{amssymb}
\usepackage{amsthm} 

\theoremstyle{plain}
\newtheorem{theorem}{Theorem}
\newtheorem{lemma}{Lemma}

\newtheorem{assumption}{Assumption}

\theoremstyle{definition}

\usepackage{todonotes}

\usepackage{xcolor}

\newcommand{\methodname}{FLARE}
\title{\methodname: Task-Agnostic Embedding Model Evaluation via Normalizing Flows}

\author{Jingzhou Jiang \and Yixuan Tang \and Yi Yang \and Kar Yan Tam \\
        The Hong Kong University of Science and Technology \\
       \{jjiang105,ytangch\}@connect.ust.hk, \{imyiyang,kytam\}@ust.hk}

\begin{document}
\maketitle
\begin{abstract}
Selecting an embedding model for a specific target corpus is difficult when task-specific labels are unavailable. Existing label-free metrics based on kernel estimators or Gaussian mixtures fail in high-dimensional spaces and produce unstable rankings. We propose \textbf{\methodname} (\textbf{F}low-based \textbf{L}abel-free \textbf{A}ssessment of \textbf{R}epresentation \textbf{E}mbeddings), which uses normalizing flows to estimate information sufficiency directly from log-likelihoods, avoiding distance-based density estimates. We give finite-sample bounds showing that the estimation error depends on the intrinsic dimension of the data manifold rather than the raw embedding dimension. On 11 datasets and eight embedders, \methodname\ attains Spearman's $\rho$ up to $0.90$ with supervised benchmarks and remains stable for high-dimensional embeddings ($d\ge 3{,}584$), where existing label-free baselines collapse.
\end{abstract}

\section{Introduction}

Recent advances in text embeddings have produced powerful semantic representation models such as Qwen3 Embedding~\citep{yang2025qwen3} and Gemini Embedding~\citep{lee2025gemini}. However, as the number of available models grows, each with different architectures, training objectives, and pretraining corpora, selecting the most suitable model for a given corpus has become increasingly challenging. The standard approach relies on supervised benchmarks like MTEB~\citep{muennighoff-etal-2023-mteb}, ranking models by their performance on annotated tasks.


This approach requires labeled data, which is often unavailable in practice. Consider deploying a retrieval system over proprietary documents such as legal contracts, medical records, or financial reports. These collections have no existing labeled query-document pairs, and creating annotations requires significant time and domain expertise. Specialized corpora may also differ substantially from public benchmarks in vocabulary, style, and topic distribution~\citep{tang-yang-2025-finmteb}. An embedding model that ranks highly on MTEB may perform poorly on domain-specific text, but without labels we cannot measure this gap. Benchmark contamination further undermines public leaderboards: as test sets appear in pretraining data, scores become inflated. This raises our central question: \textit{how can we evaluate embedding models without labels?}

Recent work has explored task-agnostic evaluation using only unlabeled corpora. One approach analyzes geometric properties of embedding models like uniformity and alignment~\citep{wang2020understanding,rudman2022isoscore}. However, these metrics measure the embedding hypersphere structure rather than semantic content. A random projection can achieve perfect uniformity while preserving no information. A more principled alternative estimates mutual information between embeddings, quantifying how much information the embedding retains. Existing implementations use non-parametric estimators like Kernel Density Estimation (KDE) or Gaussian Mixtures \citep{darrin2024embedding}. These methods suffer from the curse of dimensionality~\citep{beirlant1997nonparametric}: as embedding dimension grows (modern models often exceed $d=3{,}000$), reliable density estimation requires exponentially more data. In high-dimensional space, these estimators become unstable and fail to predict downstream performance.

In this work, we propose \textbf{F}low-based \textbf{L}abel-free \textbf{A}ssessment of \textbf{R}epresentation \textbf{E}mbeddings (\textbf{\methodname}), a framework grounded in \textbf{information-theoretic sufficiency} \citep{darrin2024embedding}. Specifically, we quantify embedding quality by measuring the reduction in uncertainty about input data given the embedding. The core innovation lies in leveraging Normalizing Flows~\citep{durkan2019neural}, deep generative models that learn invertible transformations from complex distributions to simple base densities. Flows enable exact log-likelihood estimation via the change-of-variables formula, effectively mitigating the curse of dimensionality inherent to distance-based estimators. Our finite-sample bounds provide theoretical justification: the estimation error depends on the intrinsic effective dimension of the data manifold, not the embedding dimension. This ensures that \methodname\ remains reliable when scaling to the high-dimensional embeddings of modern LLMs.

We evaluate \methodname\ on 11 datasets across four task families. It attains Spearman's $\rho$ up to $0.90$ with supervised rankings, outperforms geometry-based and information-theoretic baselines, and remains stable for high-dimensional embeddings ($d\ge 3{,}584$) where existing methods fail.

Our contributions are summarized as follows:
\begin{itemize}
    \item We introduce \methodname, a task-agnostic text embedding evaluation framework using normalizing flows to estimate information sufficiency without labeled data.
    \item We prove finite-sample generalization bounds whose dominant term scales with the intrinsic dimension of the data manifold rather than the raw embedding dimension, explaining the estimator's behaviour in high dimensions.
    \item Across 11 datasets and eight embedders, \methodname\ predicts downstream rankings more reliably than existing baselines, particularly for high-dimensional LLM-based embeddings.
\end{itemize}

\section{Related Work}

\paragraph{Embedding Evaluation.}
Modern text embedding models, often leveraging Large Language Model (LLM) architectures, provide high-dimensional representations that generalize across diverse semantic tasks without requiring task-specific fine-tuning \citep{neelakantan2022text, wang2024textembeddingsweaklysupervisedcontrastive}. Currently, the Massive Text Embedding Benchmark (MTEB) \citep{muennighoff-etal-2023-mteb} and its specialized variants like FinMTEB \citep{tang-yang-2025-finmteb} serve as the primary evaluation standards. These benchmarks aggregate performance across diverse labeled tasks including clustering, retrieval, and semantic textual similarity (STS). All of them depend on ground-truth annotations and are therefore unsuitable when labels are absent, as is typical for proprietary, dynamic, or out-of-distribution corpora. To address this setting, we propose \methodname, a task-agnostic framework that quantifies embedding quality directly from the data distribution, offering a reliable performance proxy without the need for supervision.

\paragraph{Task-Agnostic Approaches.}
Task-agnostic metrics offer a long-standing alternative to supervised benchmarks by eliminating labeling costs. Classic indices such as the Silhouette Score \citep{rousseeuw1987silhouettes} evaluate cluster separation, while more recent studies emphasize spectral and geometric properties. Examples include Effective Rank \citep{roy2007effective} for dimensionality estimation along with Uniformity \citep{wang2020understanding} and IsoScore \citep{rudman2022isoscore} for spatial distribution analysis. A primary limitation of these metrics is their focus on geometric structure rather than semantic content. Such methods often rely on global priors like isotropy, which may not represent the intrinsic low-dimensional manifold structure characteristic of text embeddings. To address this, EMIR \citep{darrin2024embedding} introduces the concept of Information Sufficiency to quantify how well one embedding model can reconstruct another. However, standard implementations of EMIR utilize Gaussian Mixture Models (GMM), which struggle with high-dimensional data and lack the generalization guarantees required for modern LLM-based embeddings.

\paragraph{Density Estimation.}
Calculating information-theoretic measures like differential entropy requires an accurate model of the underlying probability density. Traditional non-parametric approaches, notably Kernel Density Estimation (KDE), are fundamentally limited by the curse of dimensionality \citep{silverman2018density, beirlant1997nonparametric}. In high-dimensional spaces, these methods become statistically inefficient because the sample size needed to control estimation error grows exponentially with dimensionality. While neural variational estimators such as MINE \citep{ishmael2018mine} and CLUB \citep{cheng2020club} improve scalability, they optimize variational bounds instead of exact likelihoods. These bounds often suffer from a severe bias-variance trade-off, leading to loose estimates and optimization instability \citep{poole2019variational}. Normalizing Flows provide a robust alternative by learning a sequence of invertible transformations that map complex data to a simple base distribution \citep{rezende2015variational, dinh2016density}. A key advantage is their support for exact log-likelihood computation via the change-of-variables formula \citep{papamakarios2021normalizing}. Based on this property, we treat embedding evaluation as a problem of estimating probability densities directly. By integrating flows into the information-sufficiency framework \citep{darrin2024embedding}, we ensure that our metrics remain theoretically grounded and empirically stable even for high-dimensional embeddings.

\begin{figure*}[t]
    \centering
    \includegraphics[width=.99\linewidth]{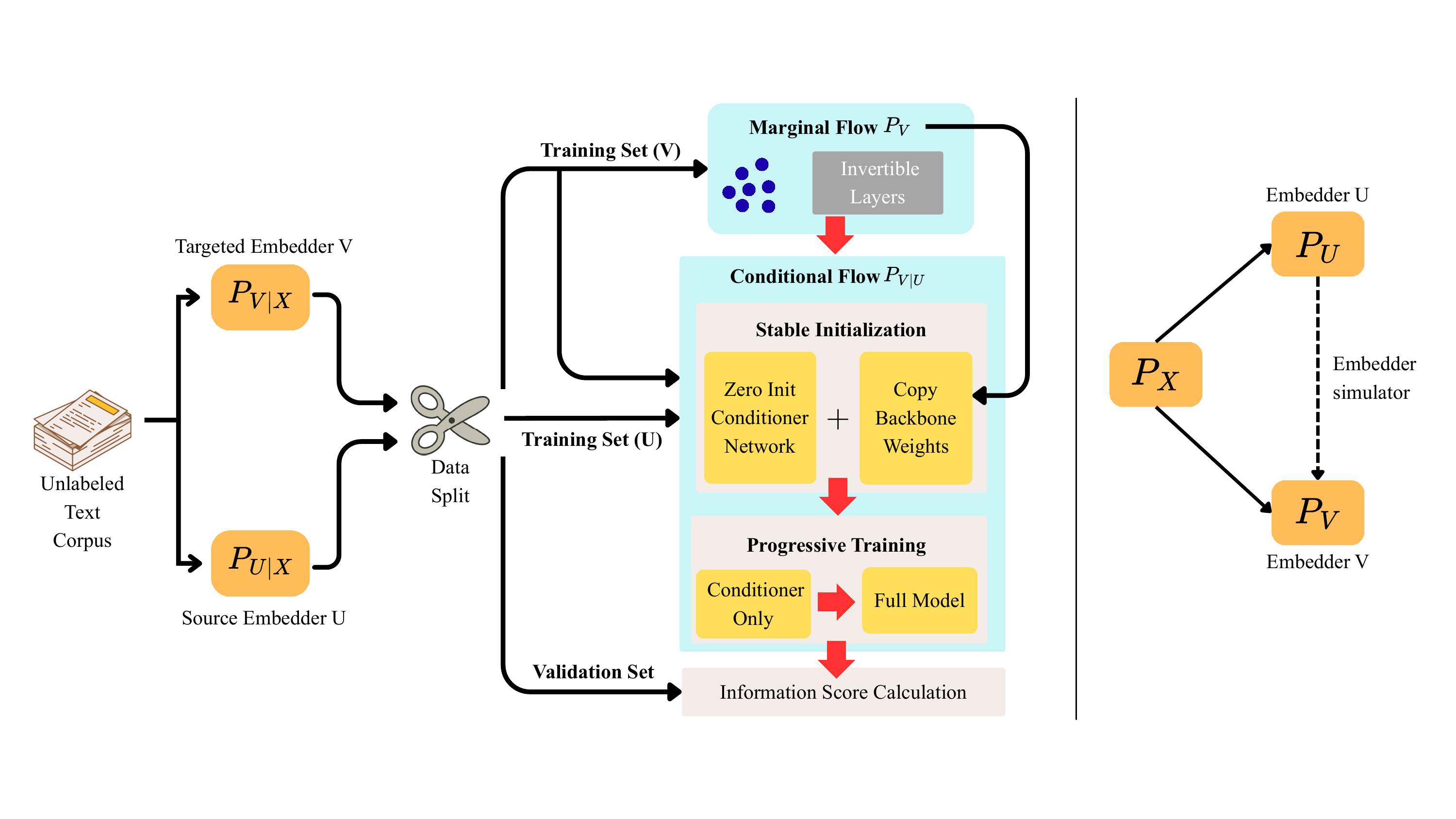}
    \caption{Overview of our two stage flow based estimation pipeline. 
\textbf{Stage 1:} Train a marginal flow $p_\phi(v)$ on target embeddings $V$ to model their intrinsic distribution. 
\textbf{Stage 2:} Initialize a conditional flow $p_\theta(v|u)$ by copying the marginal backbone weights and adding a zero-initialized low-rank conditioning branch. 
This branch is then trained to capture the dependency between source embeddings $U$ and target embeddings $V$, 
enabling computation of the information-sufficiency score via Eq.~\ref{eq:emir-is}.}
    \label{fig:framework}
\end{figure*}

\section{Method}
\label{sec:method}

To evaluate embedding quality without task-specific labels, we propose \methodname, which quantifies representation quality by measuring information sufficiency using normalizing flows.

\subsection{Problem Formulation}
\label{sec:problem}

Given an unlabeled corpus $\mathcal{X}$, we consider a set of candidate embedding models $\mathcal{E} = \{E_1, \dots, E_K\}$. Each model $E \in \mathcal{E}$ maps an input text $x \in \mathcal{X}$ to a high-dimensional representation:
\begin{equation}
    z = E(x) \in \mathbb{R}^d.
\end{equation}
To evaluate a specific model $E_a$, we pair it with a reference model $E_b$ ($b \ne a$). We denote the embedding being evaluated as the source $U$, and the reference embedding as the target $V$:
\begin{equation}
    U = E_a(x), \quad V = E_b(x).
\end{equation}
Our objective is to derive a task-agnostic score for $E_a$ based solely on these representations such that the resulting model ranking aligns with downstream supervised performance. 

\subsection{Information-Sufficiency Score}
\label{sec:is-score}

We build upon the information-sufficiency framework of \citep{darrin2024embedding}. The core intuition is that a high-quality source embedding $U$ should act as a sufficient representation of the semantic space, enabling the reconstruction of the target representations $V$. We formalize this via Information-Sufficiency ($I_s$), which measures the reduction in uncertainty of $V$ once $U$ is observed.

Let $\mathcal{F}$ be a family of marginal densities and $\mathcal{K}$ be a family of conditional densities. The directional $I_s$ score from $U$ to $V$ is defined as the difference between marginal and conditional entropy:
\begin{equation}
\label{eq:emir-is}
\begin{aligned}
\mathrm{I_s}(U \!\rightarrow\! V) &= \underbrace{\inf_{f \in \mathcal{F}} \mathbb{E}_{v}[-\log f(v)]}_{H(V)} \\
&\quad - \underbrace{\mathbb{E}_{u}\left[ \inf_{M \in \mathcal{K}} \mathbb{E}_{v \mid u}[-\log M(v \mid u)] \right]}_{H(V \mid U)}.
\end{aligned}
\end{equation}
To obtain a single quality score for model $E_a$, we compute the normalized median of its pairwise scores against all other models in the pool:
\begin{equation}
\label{eq:is_norm}
\mathrm{I_s}_{\mathrm{norm}}(E_a) = \operatorname{median}_{b \ne a} \frac{\mathrm{I_s}(U_a \!\rightarrow\! U_b)}{\dim(U_b)}.
\end{equation}
Normalization by the target dimension $\dim(U_b)$ is essential for comparability across reference models with varying output sizes, as raw entropy naturally scales with dimensionality.

\subsection{Normalizing-Flow Implementation}
\label{sec:method-flow}

We instantiate the density families in Eq.~\ref{eq:emir-is} using normalizing flows \citep{durkan2019neural}. Flows enable exact log-likelihood computation via invertible transformations, providing stable estimation.

\paragraph{Two-Stage Training.}
As shown in Figure~\ref{fig:framework}, we employ a progressive training strategy. We first train a marginal flow $p_\phi(v)$ on target embeddings to model distribution. Next, we initialize a conditional flow $p_\theta(v|u)$ by copying the marginal backbone weights. This warm-start strategy ensures the conditional model begins from a well-defined density baseline.

\paragraph{Low-Rank Conditioning.}
Standard conditional flows often use hypernetworks with $O(d^2)$ complexity, which is prohibitive for high-dimensional embeddings. We instead inject the source information $u$ through a parameter-efficient low-rank residual branch. Let $\mathbf{h}_{\text{base}}$ be the intermediate features of the target flow. The conditional feature is computed as:
\begin{equation}
    \mathbf{h}_{\text{cond}} = \mathbf{h}_{\text{base}} + B(A(u)),
\end{equation}
where $A \in \mathbb{R}^{r \times d}$ projects the $d$-dimensional source embedding to a low-rank bottleneck of dimension $r=64$, and $B$ maps from the bottleneck to the output space. 
This mechanism allows the source $u$ to adjust the transformation trajectory of the target $v$ with minimal parameter overhead.

\paragraph{Zero Initialization.}
We initialize $B$ to zeros and $A$ with random initialization. At the start of training, the conditioning term $B(A(u))$ is zero, making the conditional flow $p_\theta(v|u)$ initially equivalent to the pre-trained marginal flow $p_\phi(v)$. The dependence on the source $U$ is learned gradually, which prevents gradient instability and improves convergence speed.

\section{Theoretical Justification}
\label{sec:theory}

\paragraph{Motivation.}
In real-world deployment, embedding models must generalize to vast amounts of unseen data that extend far beyond the validation set. Purely empirical evaluation on a fixed dataset cannot theoretically guarantee reliability on the underlying population distribution.
To address this, we establish a theoretical framework relying on two pillars: the spectral stability of our flow architecture and the low-dimensional manifold hypothesis.

\paragraph{Core Assumptions.}
Our theory relies on two core assumptions (Appendix~\ref{app:ass}): a low intrinsic dimension (Assumption~\ref{ass:manifold-app}) to enable scaling to high-dimensional embeddings, and layer-wise approximate independence (Assumption~\ref{ass:indep-app}) to ensure stable gradient propagation. The latter, theoretically motivated by \citep{pmlr-v139-cohen21b}, is practically realized by our Zero-Initialization strategy, which ensures the network exhibits stable, linear growth rather than exponential instability.

\paragraph{Finite-Sample Generalization Bound.}
Building on this stability, we verify the reliability of \methodname\ by bounding the gap between the training and validation losses.

\begin{theorem}[Finite-sample generalization bound]
\label{thm:gen-bound}
Under Assumptions 1 and 2, for any fixed flow model $p_{\theta}$ and any $\delta\in(0,1)$, with probability at least $1-\delta$, the train-validation gap satisfies:
\begin{equation}
\label{eq:main-bound}
\begin{aligned}
    \bigl| \hat L_{\mathrm{val}}(\theta) - \hat L_{\mathrm{train}}(\theta) \bigr|
    &\le
    \frac{
      2 \tilde C_{\mathrm{Rad}}\,
      L\,\bar\sigma\,
      \sqrt{d_{\mathrm{eff}}}
    }{
      \sqrt{m}
    }
    \\
    &\quad
    +
    M_{\mathrm{val}}
    \sqrt{
      \frac{
        \log(2/\delta)
      }{
        2 m_{\mathrm{val}}
      }
    }
    \\
    &\quad
    +
    3 M_{\mathrm{train}}
    \sqrt{
      \frac{
        \log(2/\delta)
      }{
        2 m
      }
    } .
\end{aligned}
\end{equation}
\end{theorem}

\paragraph{Interpretation.}
The bound scales with the intrinsic dimension $d_{\text{eff}}$ rather than the raw embedding dimension $d$. Because $d_{\text{eff}}\ll d$ for semantic representations, reliable evaluation is possible on high-dimensional embeddings with moderate sample sizes. The error also depends linearly on depth $L$ and spectral stability $\bar{\sigma}$, both kept small by our zero-initialization strategy.

\section{Experiments}
We design our experiments to evaluate the empirical utility of the Flow-based estimator across three primary dimensions:
\textbf{(Q1)} its reliability in predicting model rankings across diverse downstream tasks; 
\textbf{(Q2)} its comparative performance against kernel-based baselines in high-dimensional embedding spaces; 
and \textbf{(Q3)} the alignment between empirical observations and our theoretical generalization guarantees.

\subsection{Embedders and Datasets}

\paragraph{Embedding Models.}
Our evaluation encompasses eight representative embedding models, covering a broad spectrum of architectures and dimensionalities. We focus specifically on high-dimensional space, including BGE-Multilingual-Gemma2 ($d=3{,}584$) \citep{bge-m3}, gte-Qwen2-7B-instruct ($d=3{,}584$) \citep{li2023towards}, and four Mistral-7B-based models ($d=4{,}096$): Zeta-Alpha-E5-Mistral \footnote{\url{https://huggingface.co/zeta-alpha-ai/Zeta-Alpha-E5-Mistral}}, GritLM-7B \citep{muennighoff2024generative}, SFR-Embedding-Mistral \citep{SFRAIResearch2024}, and Linq-Embed-Mistral \citep{LinqAIResearch2024}. For lower-dimensional benchmarks, we include stella-base-en-v2 \footnote{\url{https://huggingface.co/infgrad/stella-base-en-v2}} ($d=768$) and all-MiniLM-L6-v2 \citep{reimers-gurevych-2019-sentence} ($d=384$).

\paragraph{Dataset Selection and Task Categorization.}
To approximate the realistic setting in which models are evaluated on novel internal corpora, we deviate from standard benchmarks such as MTEB. Public benchmarks routinely leak into pre-training data and inflate reported scores, so we apply a temporal filter and select 11 Hugging Face datasets released \textit{after} the training cutoff of every candidate model. \methodname\ itself is unsupervised; ground-truth labels are used only as an oracle to validate the predicted rankings. Dataset statistics are reported in Appendix~\ref{app:detail}. The 11 datasets cover four task categories:
\begin{itemize}
    \item \textbf{Classification:} Apt-eval \citep{saha-feizi-2025-almost} (safety/robustness), GT-FintechLab \citep{shah2025wordsuniteworldunified} (finance), and BhashaBench-Finance \citep{devane2025bhashabenchv1comprehensivebenchmark} (multilingual finance).
    \item \textbf{Retrieval:} AIR-Bench-Finance \citep{chen2024airbench}, LIMIT \citep{weller2025theoreticallimit} (instruction-following), and ArXiv-Abstracts 2025 \footnote{\url{https://huggingface.co/datasets/almanach/arxiv_abstracts_2025}} (scientific literature). For retrieval tasks, we embed the passage corpus only (not queries) to simulate the label-free setting.
    \item \textbf{Semantic Textual Similarity (STS):} Augmented STS-B \footnote{\url{https://huggingface.co/datasets/maiammar/Augmented_stsb_multi_mt}}, LivNLP-STS \citep{zhang2025annotatingtrainingdataconditional}, and Philosophical-STS \footnote{\url{https://huggingface.co/datasets/johnnyboycurtis/Philosophical-STS-Text-Pairs}}.
    \item \textbf{Clustering:} Clustered-FunPang Medical\footnote{\url{https://huggingface.co/datasets/mukulb/clustered_FUNPANG_dataset_with_groups}}, and Reasoning-Clustering \footnote{\url{https://huggingface.co/datasets/Ibisbill/Clustering_deduplicated_reasoning}}.
\end{itemize}

\paragraph{Evaluation Protocol.}
We assess the reliability of our unsupervised Information Sufficiency (IS) metric by measuring its alignment with ground-truth supervised rankings. Ground-truth performance is established using standard MTEB metrics \citep{muennighoff-etal-2023-mteb}: F1 macro for classification, nDCG@10 for retrieval, Spearman correlation for STS, and V-measure for clustering. We quantify ranking alignment using Spearman’s rank correlation ($\rho$) and Pearson correlation ($r$) between the predicted IS scores and the supervised metrics.

\paragraph{Baselines.} 
We benchmark \methodname\ against existing unsupervised metrics, explicitly categorized into two types of unsupervised evaluation methods: geometric-based and information-theoretic. 
For the geometric-based methods, we considered (1) \textbf{Uniformity}~\citep{wang2020understanding} and (2) \textbf{IsoScore}~\citep{rudman2022isoscore}; 
and (3) \textbf{Silhouette Score}~\citep{rousseeuw1987silhouettes}.
For the information-theoretic methods, we considered (4) \textbf{EMIR}~\citep{darrin2024embedding}.
Our work aligns with the information-theoretic evaluation (specifically the framework established by EMIR), as we share the fundamental objective of quantifying representation quality via information retention. However, we diverge by employing normalizing flows to robustly estimate densities in high-dimensional spaces where the GMMs used in EMIR may fail.

\subsection{Main Results}
\label{subsec:results}


\begin{table*}[t]
\centering
\renewcommand{\arraystretch}{1.2}
\resizebox{\textwidth}{!}{%
    \begin{tabular}{l | ccccc | ccccc}
    \toprule
    & \multicolumn{5}{c|}{\textbf{Spearman's $\rho$}} & \multicolumn{5}{c}{\textbf{Pearson's $r$}} \\
    \cmidrule(lr){2-6} \cmidrule(lr){7-11}
    \textbf{Task} & \textbf{Uni.} & \textbf{Iso.} & \textbf{Sil.} & \textbf{EMIR} & \textbf{Ours} & \textbf{Uni.} & \textbf{Iso.} & \textbf{Sil.} & \textbf{EMIR} & \textbf{Ours} \\
    \midrule
    Class. & 0.18 & -0.40 & -0.06 & -0.06 & \textbf{0.56} & \textbf{0.35} & -0.58 & -0.08 & -0.20 & 0.31 \\
    STS & 0.01 & -0.33 & 0.56 & -0.06 & \textbf{0.70} & \textbf{0.70} & -0.14 & 0.53 & -0.08 & \textbf{0.70} \\
    Retr. & 0.08 & -0.45 & -0.03 & -0.22 & \textbf{0.72} & 0.43 & -0.37 & 0.39 & -0.18 & \textbf{0.64} \\
    Clust. & -0.14 & 0.29 & -0.24 & -0.16 & \textbf{0.83} & -0.43 & 0.05 & -0.25 & -0.10 & \textbf{0.69} \\
    \midrule
    \textbf{Avg} & 0.05 & -0.27 & 0.08 & -0.12 & \textbf{0.69} & 0.33 & -0.29 & 0.18 & -0.14 & \textbf{0.58} \\
    \bottomrule
    \end{tabular}%
}
\caption{Task-aggregated Comparison with Unsupervised Baselines. }
\label{tab:aggregated_results}
\end{table*}

Table~\ref{tab:aggregated_results} summarizes the ranking correlations across eleven representative datasets, revealing three key findings. Detailed per-dataset results are shown in Appendix~\ref{app:detail}.

\paragraph{Flow-Based Estimation Succeeds Where Kernel Methods Fail.}
\methodname\ achieves an average Spearman correlation of $\rho=0.70$, outperforming all unsupervised baselines. EMIR, which shares our information-sufficiency framework but relies on Gaussian Mixture Model (GMM) density estimation, yields a negative average correlation ($\rho=-0.12$), i.e.\ systematically inverted rankings: in high-dimensional spaces ($d\ge 3{,}584$) kernel densities become vanishingly sparse and distance-based estimates unreliable. Our flow-based approach replaces these estimates with an explicit parametric density that adapts to the intrinsic manifold, and maintains positive correlations across all task categories.

\paragraph{Geometric Baselines Exhibit Inconsistent Performance.}
Table~\ref{tab:aggregated_results} reveals that existing unsupervised metrics struggle to maintain consistent correlations across task types.
Uniformity achieves near-zero average correlation ($\rho=0.05$), suggesting that embedding space uniformity alone is not predictive of downstream quality.
IsoScore exhibits negative correlations across nearly all categories (Avg $\rho=-0.27$), as it penalizes anisotropy under the assumption that uniformity maximizes information capacity. However, high-quality semantic spaces are inherently anisotropic—meaningful concepts naturally form dense, non-uniform clusters on the manifold rather than populating the hypersphere uniformly.
Silhouette shows task-specific success only on STS ($\rho=0.56$) but fails elsewhere, indicating that geometric cohesion does not generalize as a universal quality indicator.
As Figure~\ref{fig:stability} shows, these rigid geometric assumptions produce high variance and frequent ranking inversions. The information-theoretic objective behind \methodname\ instead adapts to the intrinsic data density rather than imposing a target shape on it, enabling \methodname\ to remain aligned with ground truth across heterogeneous task families.
\begin{figure}[htbp]
    \centering
    \includegraphics[width=\linewidth]{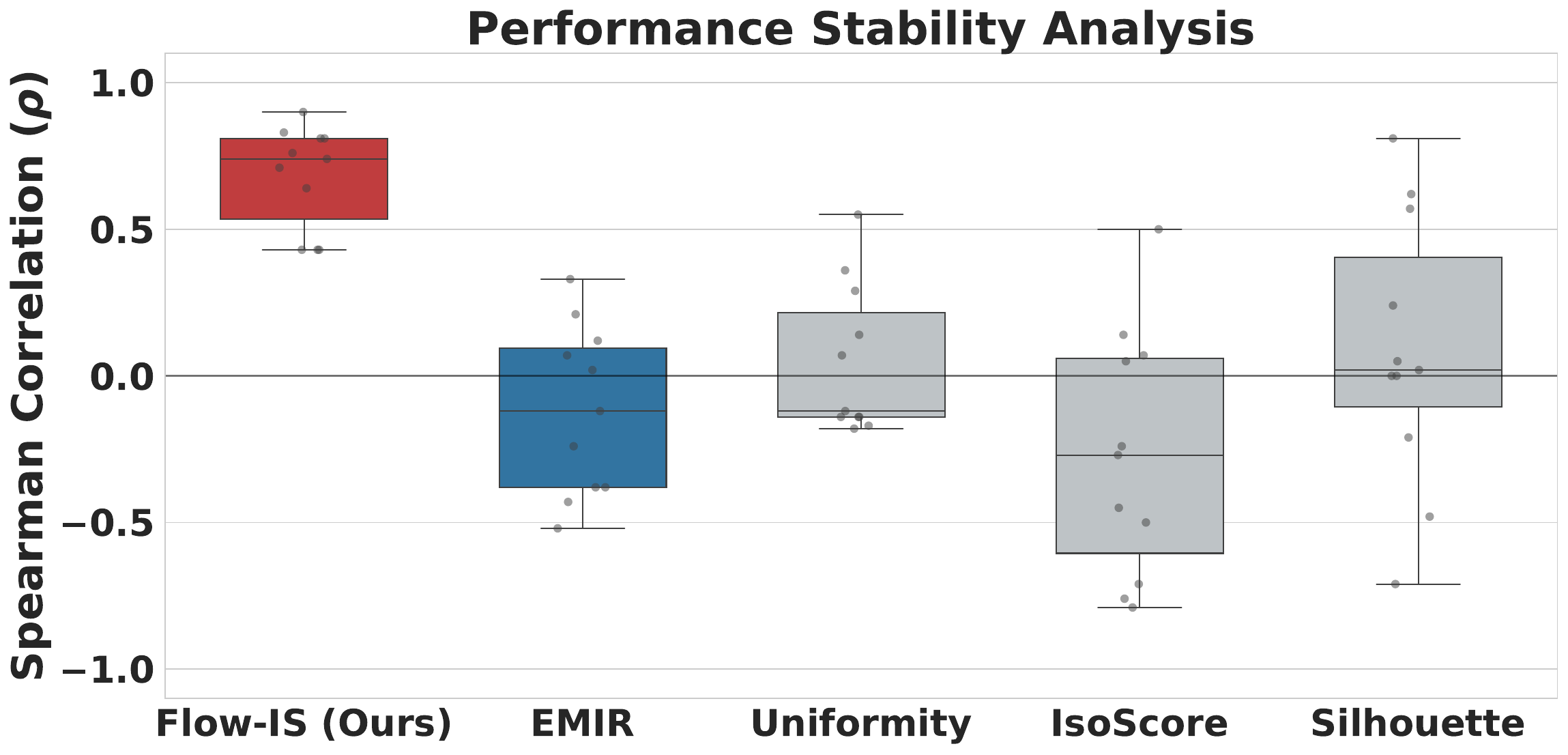}
    \caption{\textbf{Stability Analysis.} Distribution of Spearman correlations across all datasets. Geometric baselines (grey) and kernel-based EMIR (blue) exhibit high variance and frequent negative correlations, while \methodname\ (red) stays positively aligned with ground truth on every dataset.}
    \label{fig:stability}
\end{figure}

\paragraph{Consistent Performance Across Task Families.}
A useful label-free metric must rank embedders across heterogeneous downstream tasks, not just on a single family. Existing baselines fail this requirement: each one has at least one task family on which it correlates negatively with ground truth (Table~\ref{tab:aggregated_results}). \methodname\ correlates positively on every family, with $\rho=0.83$ on Clustering, $\rho=0.72$ on Retrieval, $\rho=0.70$ on STS, and $\rho=0.56$ on Classification, and the per-dataset breakdown in Table~\ref{tab:main_results} (Appendix~\ref{app:detail}) contains no negative entries.
The signal is strongest on STS, Retrieval, and Clustering, which score embeddings directly through similarity in the embedding space. Classification involves an additional supervised predictor that can exploit features beyond the distributional structure captured by an unsupervised metric, which caps the achievable correlation.

\paragraph{Ranking Stability under Subsampling.}
We probe sample-size sensitivity by recomputing $I_s$ on subsamples of the validation set at ratios $\alpha\in\{0.05, 0.1, \dots, 1.0\}$ and measuring the deviation $\Delta_\rho(\alpha){=}|\rho(\alpha){-}\rho(1.0)|$ from the full-data ranking.
For 8 of the 11 datasets, $\Delta_\rho<0.05$ even at $\alpha{=}0.2$; the remaining three datasets are STS corpora where rank-based metrics are inherently more sample-sensitive, and even there $\Delta_\rho<0.07$ once $\alpha{\ge}0.2$.
The flow therefore recovers the global manifold from far fewer observations than non-parametric estimators require, suggesting that 20–40\% of the validation set is sufficient for reliable ranking on small, specialized corpora.
Per-dataset stability curves are reported in Appendix~\ref{app:ranking-stability}.

\subsection{Generalization Bound Analysis}
\label{subsec:bound_analysis}


\begin{table}[t]
\centering
\small
\renewcommand{\arraystretch}{1.1}
\setlength{\tabcolsep}{10pt}
\begin{tabular}{l c c}
\toprule
\textbf{Task Type} & \textbf{Bound Ratio} & \textbf{Rademacher \%} \\
\midrule
Classification & $11.0\times$ & $92.4\%$ \\
STS & $21.1\times$ & $94.5\%$ \\
Retrieval & $21.2\times$ & $95.5\%$ \\
Clustering & $18.4\times$ & $98.5\%$ \\
\midrule
\textbf{Average} & \textbf{17.9$\times$} & \textbf{95.2\%} \\
\bottomrule
\end{tabular}
\caption{\textbf{Theoretical validation.} Bound Ratio ($\Delta_{\mathrm{theo}}/\Delta_{\mathrm{emp}}$) and Rademacher complexity contribution, grouped by task type.}
\label{tab:theoretical_compact}
\end{table}

To empirically validate the guarantee provided by Theorem~\ref{thm:gen-bound}, we compare the derived theoretical bound against the observed empirical generalization gap. This validation is crucial for addressing the practical challenges of enterprise deployment, ensuring that the evaluation method generalizes reliably as new data continuously arrives.

\paragraph{Setup.}
Since the Information Sufficiency estimator decomposes into marginal and conditional components, its total estimation error is bounded by the sum of their respective generalization gaps. 
We therefore compute the empirical gap as the sum of the absolute differences between training and validation Negative Log-Likelihoods (NLL) for both flows, and compare this empirical quantity against the theoretical bound derived from our model architecture and sample complexity.

\paragraph{Analysis.}
As reported in Table~\ref{tab:theoretical_compact}, the theoretical bound $\Delta_{\mathrm{theo}}$ upper-bounds the empirical gap $\Delta_{\mathrm{emp}}$ by a margin ranging from $11.0\times$ to $21.2\times$ across task types (average ratio $17.9\times$).
This confirms that our conservative linear bound remains informative, providing meaningful generalization guarantees without being vacuously loose.
The Rademacher complexity term accounts for $92.4\%$ to $98.5\%$ of the total bound (average $95.2\%$), which aligns with its dependence on the effective intrinsic dimension $d_{\mathrm{eff}}$. As expected, higher-dimensional embedding spaces incur larger complexity penalties.
Notably, classification tasks exhibit tighter bounds ($11.0\times$) compared to retrieval and STS tasks ($\sim\!21\times$), a discrepancy that reflects the richer representational capacity required for fine-grained semantic matching in the latter.
Practically, these findings certify that \methodname\ provides a trustworthy and theoretically grounded signal on unseen data.


\begin{figure*}[htbp]
\centering
\includegraphics[width=\linewidth]{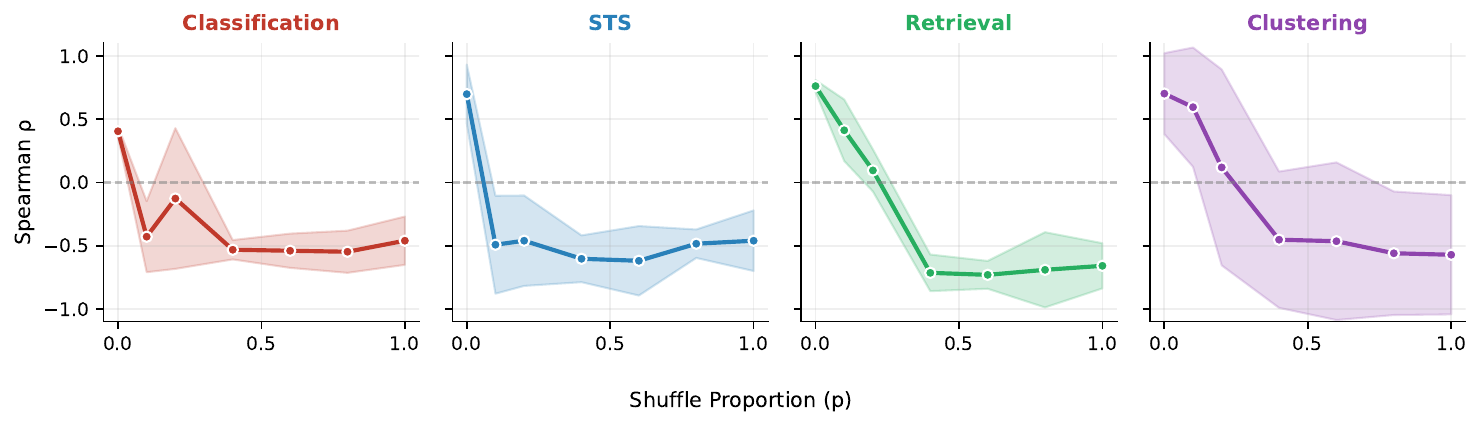}
\caption{Partial shuffle ablation by task type. Increasing the shuffle proportion $p$ causes correlation to degrade from positive ($p=0$) to negative ($p=1$), confirming that the metric relies on correct alignment.}
\label{fig:partial_shuffle}
\end{figure*}

\subsection{Ablation Study}
\label{subsec:ablation}

\paragraph{Shuffle ablation.}
We partially shuffle the correspondence between source embeddings $U$ and target embeddings $V$ at a ratio $p \in [0, 1]$ while keeping the remaining pairs unchanged. 
As visualized in Figure~\ref{fig:partial_shuffle} and detailed in Appendix~\ref{sec:shuffle}, the Spearman correlation with downstream performance degrades significantly as $p$ increases, transitioning from positive to negative around $p=0.2$ to $0.4$. 
The degradation pattern varies across task types: Retrieval and Clustering exhibit a gradual decline, whereas Classification and STS show sharper transitions at lower shuffle ratios. 
At full shuffle ($p=1.0$), all categories converge to negative correlations (averaging $\rho \approx -0.4$). 
This confirms that our metric relies heavily on the correct \textbf{alignment between source and target embeddings} rather than marginal statistics, and the varying sensitivity across tasks may reflect differences in the underlying embedding structure.


\paragraph{Conditional-only ablation.}
We investigate the contribution of the marginal term by comparing the full metric against a conditional-only variant, defined as $I_{\text{cond}}(u,v)=\log p_{\phi}(v\mid u)$. 
While this variant may be competitive when the target marginal distribution varies little, it lacks consistent reliability across diverse tasks. 
As quantified in Table~\ref{tab:ablation_cond_only} and visualized in Figure~\ref{fig:cond_only} (see Appendix~\ref{sec:ablation_marginal}), removing the marginal term $\mathbb{E}[\log p_\phi(v)]$ reduces the average Spearman correlation substantially from 0.70 to 0.21. 
Both components are therefore necessary: the conditional term captures the cross-model mapping, the marginal term the intrinsic structure of the target space. We use the full score as default.

\paragraph{Aggregation strategy.}
We aggregate the $N{-}1$ per-pair $I_s$ scores of each source model into a single number using the median. The per-pair distribution is heavy-tailed: a few targets with very different geometry (e.g.\ tokenizer or pooling mismatches) inflate the mean and add noise unrelated to the source model's quality. Table~\ref{tab:ablation_aggregation} compares the mean, the median, and a $20\%$ trimmed mean, and the median attains the highest or near-highest Spearman correlation with ground truth on every task family. The trimmed mean closes most of this gap but introduces an extra hyperparameter, so we keep the median as the default.

\subsection{Additional Analysis}
\label{subsec:additional_analysis}

\paragraph{Assumption validation.}
Our theoretical guarantees in Theorem~\ref{thm:gen-bound} are established under Assumption~\ref{ass:indep-app}, which requires that the Jacobians of successive coupling layers act along approximately orthogonal directions. Under this condition, the overall Lipschitz constant grows approximately linearly with depth rather than exponentially (Lemma~\ref{lem:arch-stab-full}, Eq.~\ref{eq:lip-bound-full}). Although this assumption is idealized, it is supported by empirical evidence from three validation experiments; full results are reported in Appendix~\ref{app:assumption_validation}.
First, the singular vectors of inter-layer Jacobians exhibit a mean cosine similarity of $0.010$, which is below the random baseline $1/\sqrt{d}=0.016$, indicating near-orthogonality between layers. This behavior is consistent with the alternating binary masks and random permutations used in the NSF architecture. Second, the average per-layer spectral amplification is $1.049$, showing that each layer contributes only a small and controlled perturbation, as encouraged by our zero-initialization strategy. Third, the cumulative perturbation amplification across all 18 atomic transforms in the composite flow (the 6 coupling layers, each followed by ActNorm and a random permutation) is only $2.38\times$, which is consistent with the approximately linear Lipschitz growth implied by Assumption~\ref{ass:indep-app}.
Overall, these results indicate that the structural regime required by Assumption~\ref{ass:indep-app} is well matched to the behavior of the trained models in practice.

\paragraph{Top-3 model identification.}
For practical model selection, the top of the ranking matters more than the tail: a practitioner mostly needs to know which two or three candidates to fine-tune or deploy, not the relative order of the weakest ones. We therefore compare the top-3 models ranked by $I_s$ against the supervised top-3 on each of the 11 datasets. \methodname\ recovers the exact top-3 set on $4/11$ datasets and matches at least $2$ out of $3$ on $10/11$ datasets (the per-dataset breakdown is given in Appendix~\ref{app:top3}). The single dataset on which only $1/3$ models agree is BhashaBench, where the top three supervised scores are all within one accuracy point and the ordering is essentially noise. In other words, whenever the supervised signal cleanly separates a leading group, \methodname\ identifies that group on its own, so the estimator is precise enough to drive model selection on an unlabeled corpus without any task-specific supervision.

\paragraph{Statistical robustness.}
To quantify ranking stability under a small pool ($N{=}8$), we run a leave-one-out bootstrap on the per-dataset $8\times8$ IS matrix. As reported in Table~\ref{tab:bootstrap_ci}, $\rho_{\min}$ remains positive on every dataset, and the resampled Spearman $\rho$ aggregates to $[0.583, 0.843]$, so the positive-correlation finding does not hinge on the inclusion of any particular model (full protocol in Appendix~\ref{app:bootstrap_ci}). As a complementary test, converting rankings into $11\times\binom{8}{2}{=}308$ pairwise preferences, \methodname\ matches the supervised ordering on $73.1\%$ of pairs ($225/308$; binomial $p{=}1.37\times 10^{-16}$, 95\% CI lower bound $68.6\%$), well above the $50\%$ chance baseline. Most residual disagreement comes from pairs with near-identical downstream performance, while the $90.9\%$ Top-3 overlap confirms that \methodname\ reliably separates strong from weak models, the primary goal of model selection.

\begin{table}[htbp]
\centering
\renewcommand{\arraystretch}{1.05}
\setlength{\tabcolsep}{6pt}
\begin{tabular}{l c c}
\toprule
\textbf{Dataset} & \textbf{Spearman $\rho$} & \textbf{Bootstrap range} \\
\midrule
FunPang        & 0.905 & [0.857, 0.964] \\
Aug-STSB       & 0.833 & [0.750, 0.964] \\
LivNLP-STS     & 0.833 & [0.750, 0.964] \\
LIMIT          & 0.810 & [0.714, 0.893] \\
BhashaBench    & 0.802 & [0.739, 0.919] \\
arXiv '25      & 0.762 & [0.643, 0.857] \\
Reasoning      & 0.762 & [0.643, 0.857] \\
AIR-Bench      & 0.714 & [0.607, 0.857] \\
apt-eval       & 0.429 & [0.321, 0.750] \\
gtfintechlab   & 0.429 & [0.143, 0.607] \\
Philo-STS      & 0.429 & [0.250, 0.643] \\
\midrule
\textbf{Average} & \textbf{0.701} & \textbf{[0.583, 0.843]} \\
\bottomrule
\end{tabular}
\caption{Leave-one-out bootstrap over the 8-model pool. ``Spearman $\rho$'' is the full 8-model estimate (matching Table~\ref{tab:main_results}); ``Bootstrap range'' is $[\rho_{\min}, \rho_{\max}]$ across the 8 resampled replicates obtained by recomputing $\rho$ on each 7-model subset. The lower bound stays positive on every dataset.}
\label{tab:bootstrap_ci}
\end{table}

\paragraph{Training stability.}
To test whether \methodname\ relies on a brittle optimum, we add Gaussian noise to each parameter tensor at $\sigma{=}1{-}20\%$ of its mean absolute value and re-evaluate the held-out NLL. Across the 616 conditional flows trained for the 11 datasets, the median relative NLL change is below $0.03\%$ at $\sigma{\le}5\%$ and 1.22\% at $\sigma{=}20\%$; the median–mean gap indicates that the few sensitive pairs are outliers rather than the rule. The trained flows therefore lie in flat basins, and the resulting rankings are not driven by training randomness (Appendix~\ref{app:perturbation}).

\section{Conclusion}
We presented \methodname, a label-free evaluator of text embedding models that estimates information sufficiency with normalizing flows. The framework replaces the kernel-density backbone of prior information-theoretic estimators with an exact log-likelihood from a normalizing flow, removing the dimensional sparsity that destabilises kernel methods on modern LLM-based embeddings.
Finite-sample bounds derived via Rademacher complexity show that the estimation error scales with the intrinsic dimension of the data manifold rather than the raw embedding dimension, accounting for the estimator's behaviour in high-dimensional regimes.
Across 11 datasets and eight embedders, \methodname\ correlates positively with supervised rankings on every task family (average Spearman $\rho{=}0.70$, up to $0.90$), agrees with supervised pairwise preferences on $73.1\%$ of pairs, and recovers $\geq 2/3$ of the supervised top-3 on $10/11$ datasets. The induced ranking is also stable: the leave-one-out bootstrap keeps $\rho_{\min}>0$ on every dataset, and 20--40\% of the validation data is sufficient to reproduce the full-data ranking. Together, these properties make \methodname\ a practical alternative to annotated benchmarks for model selection on unlabeled corpora.

\paragraph{Practitioner's guide.}
The intended workflow is label-free model selection: a practitioner curates an unlabeled target corpus, picks a small pool of $N$ candidate embedders, computes pairwise $I_s$ on the corpus, and selects the top-scoring candidates. No downstream task or annotation is required. $I_s$ is a relative quantity, comparable within a single dataset and candidate pool but not across datasets; the per-dimension normalization in Eq.~\ref{eq:is_norm} places encoders of differing widths on a common scale. The induced ranking is reliable when the per-dimension gap between adjacent candidates exceeds the leave-one-out bootstrap range (Appendix~\ref{app:bootstrap_ci}); for tightly clustered candidates, the Top-3 overlap is a more robust summary than the full Spearman $\rho$. When a candidate exhibits an unexpectedly low score, the decomposition $I_s = H(V) - H(V\mid U)$ both localises the cause and indicates a remedy. A low marginal $H(V)$ reflects a poorly structured target density (e.g., collapsed or anisotropic), and motivates either substituting the candidate with a higher-capacity encoder or applying anisotropy-reducing post-processing such as whitening. A low conditional $H(V\mid U)$ accompanied by a healthy marginal indicates weak cross-model transferability, for which lightweight alignment such as a learned linear projection or distillation from a stronger reference is typically sufficient. Before attributing a depressed score to embedding quality, however, one should verify that the marginal flow's validation NLL has converged, as under-fitting biases $I_s$ downward for reasons unrelated to representation quality.

\section*{Limitations}

\methodname\ has three main limitations. 
First, as a learning-based evaluation method, \methodname\ incurs non-trivial training cost: one conditional flow on Reasoning ($d{=}3{,}584$) takes ${\sim}1.33$ GPU hours on a V100, and the full $N{=}8$ pairwise evaluation costs ${\sim}74.5$ GPU hours. Caching the marginal flow makes adding a new model $O(N)$ (${\sim}10$ GPU hours) rather than $O(N^2)$, and we further mitigate inference latency by using discrete normalizing flows instead of continuous-time approaches like flow matching; however, training a generative model remains a bottleneck for resource-constrained scenarios. Training-free baselines such as KDE-based EMIR cost only ${\sim}6$ GPU hours, though they degrade in high-dimensional spaces ($d{\ge}3{,}584$) where density estimates can collapse to a single kernel component.

Second, we observe a performance discrepancy across task types. 
While \methodname\ excels on geometry-centric tasks like STS and Retrieval, its correlation is lower for Classification ($\rho=0.56$). 
This suggests that global information sufficiency captures the overall semantic manifold but may miss fine-grained, class-specific features required for linear separability.

Third, our theoretical analysis relies on simplifying assumptions. Specifically, regarding Assumption~\ref{ass:indep-app}, parameter coupling during end-to-end backpropagation inevitably introduces inter-layer correlations, making the approximate independence assumption an idealization that is computationally intractable to verify in high dimensions.

\bibliography{custom}

\appendix

\section{Theoretical Assumptions and Proofs}
\label{Theoretical}

\subsection{Problem setup}

Let $\mathcal{D}$ be an unknown distribution on $\mathbb{R}^d$. We observe an i.i.d.\ training sample
\begin{equation}
  S_{\mathrm{train}}
  =
  \{v^{\mathrm{train}}_1,\dots,v^{\mathrm{train}}_{m}\}
  \sim \mathcal{D}^m .
\end{equation}
A normalizing flow $T_\theta$ induces a density
\begin{equation}
  p_\theta(v)
  =
  p_Z\bigl(T_\theta(v)\bigr)\,
  \bigl|
    \det J_{T_\theta}(v)
  \bigr| ,
\end{equation}
where $p_Z$ is a fixed base density and $J_{T_\theta}(v)$ is the Jacobian of the flow transformation.
We use the negative log-likelihood loss $\ell_\theta(v) = - \log p_\theta(v)$.
The empirical training risk is
\begin{equation}
  \hat L_{\mathrm{train}}(\theta)
  =
  \frac{1}{m}
  \sum_{i=1}^{m}
  \ell_\theta\bigl(v^{\mathrm{train}}_i\bigr) .
\end{equation}
The population risk is
\begin{equation}
  L(\theta)
  =
  \mathbb{E}_{v\sim\mathcal{D}}
  \bigl[
    \ell_\theta(v)
  \bigr] .
\end{equation}
Given a validation set
\begin{equation}
  S_{\mathrm{val}}
  =
  \{v^{\mathrm{val}}_1,\dots,v^{\mathrm{val}}_{m_{\mathrm{val}}}\}
  \sim \mathcal{D}^{m_{\mathrm{val}}},
\end{equation}
the empirical validation risk is
\begin{equation}
  \hat L_{\mathrm{val}}(\theta)
  =
  \frac{1}{m_{\mathrm{val}}}
  \sum_{j=1}^{m_{\mathrm{val}}}
  \ell_\theta\bigl(v^{\mathrm{val}}_j\bigr) .
\end{equation}
We aim to bound the generalization gap between the validation risk (observable proxy for performance) and the training risk (optimization objective):
\begin{equation}
  \Delta(\theta)
  =
  \bigl|
    \hat L_{\mathrm{val}}(\theta) - \hat L_{\mathrm{train}}(\theta)
  \bigr|.
\end{equation}
Note that bounding this gap involves controlling the deviation of both empirical risks from the population risk $L(\theta)$.

We first consider a marginal flow for a fixed embedder, then apply the same argument to a conditional flow.
The Information Sufficiency ($I_s$) score is defined as
\begin{equation}
  IS
  =
  L_{\mathrm{marg}}(V)
  -
  L_{\mathrm{cond}}(V \mid U) ,
\end{equation}
and the empirical $\widehat{IS}$ score replaces the population risks with their validation counterparts.
Establishing a finite-sample bound on the generalization gaps of the marginal and conditional flows directly yields a bound for the estimation error of the $I_s$ score.

\subsection{Core assumptions}
\label{app:ass}

\begin{assumption}[Low intrinsic dimension]
\label{ass:manifold-app}
The embedding distribution is supported on a compact subset $\mathcal{M} \subset \mathbb{R}^d$ with intrinsic dimension $d_{\mathrm{eff}} \ll d$ and bounded diameter.
This type of low-dimensional structure for learned representations is consistent with empirical measurements of intrinsic dimensionality in deep features \citep{ansuini2019intrinsic}.
\end{assumption}

\begin{assumption}[Approximate layer independence]
\label{ass:indep-app}
The flow $T_\theta$ is a composition of $L$ invertible blocks with \textbf{perturbation rank} at most $r$. 
Assuming a regime of near-identity mappings (e.g., via residual connections or zero-initialization), the Jacobians of different blocks are approximately independent in their dominant singular directions. 
Consequently, the overall Lipschitz behaviour of the composition \textbf{grows linearly} in depth rather than exhibiting exponential growth. 
This assumption is consistent with analyses of signal propagation and dynamical isometry in deep architectures~\citep{pmlr-v139-cohen21b}.
\end{assumption}

\paragraph{Remark: Validity via Zero-Initialization.} 
While assuming approximate independence between layer Jacobians is non-trivial for generic deep networks, our Zero-Initialization strategy (Section~\ref{sec:method-flow}) provides a rigorous structural justification. 
Specifically:
(1) \textbf{Identity Start}: By initializing the conditioner's projection matrix $B$ to zero, the conditional flow starts as an identity transformation relative to the backbone ($T_{\theta} = T_{\phi}$), ensuring that initial layer-wise Jacobians satisfy $J_l = I$.
(2) \textbf{Linear Accumulation}: During progressive training, the deviation $\Delta_l$ from identity (where $J_l = I + \Delta_l$) is explicitly constrained via $L_2$ regularization. In this small-perturbation regime, the norm of the composed Jacobian follows a first-order approximation $\|\prod_{l=1}^L (I + \Delta_l)\| \approx \|I + \sum_{l=1}^L \Delta_l\|$, which implies that the Lipschitz constant grows \textbf{linearly} with depth $L$ rather than exponentially.
This design effectively aligns the flow's architectural behavior with the stability postulated in Assumption~\ref{ass:indep-app}.

\subsection{Rademacher complexity on a manifold}

Let $\mathcal{F}$ be a class of real-valued functions on $\mathcal{M}$.
Given a sample $S = \{v_1,\dots,v_m\} \subset \mathcal{M}$, the empirical Rademacher complexity of $\mathcal{F}$ is
\begin{equation}
  \mathcal{R}_m(\mathcal{F})
  =
  \mathbb{E}_\sigma
  \biggl[
    \sup_{f \in \mathcal{F}}
    \frac{1}{m}
    \sum_{i=1}^m
    \sigma_i f(v_i)
  \biggr] ,
\end{equation}
where $\sigma_i$ are i.i.d.\ Rademacher variables.

\begin{lemma}[Rademacher complexity on a manifold]
\label{lem:manifold-rad-full}
Let $\mathcal{F}$ be the class of \textbf{linear functions} $f(v) = \langle w, v \rangle$ restricted to $\mathcal{M}$, where $\|w\| \le L_f$.
Under Assumption~\ref{ass:manifold-app}, there exists a constant
$C_{\mathrm{Rad}} > 0$ such that
\begin{equation}
  \mathcal{R}_m(\mathcal{F})
  \;\le\;
  \frac{
    C_{\mathrm{Rad}}\,
    L_f\,D\,
    \sqrt{d_{\mathrm{eff}}}
  }{
    \sqrt{m}
  } .
  \label{eq:manifold-rad-full}
\end{equation}
\end{lemma}

\begin{proof}
The manifold $\mathcal{M}$ admits a covering number bound of the form
\begin{equation}
  \mathcal{N}(\mathcal{M},\varepsilon)
  \;\le\;
  C_0
  \left(
    \frac{D}{\varepsilon}
  \right)^{d_{\mathrm{eff}}}
  \quad
  \text{for all } \varepsilon > 0 .
\end{equation}
Here, $\mathcal{N}(\mathcal{M}, \varepsilon)$ represents the covering number of the manifold, and $C_0$ is a geometry-dependent constant.

For the class of linear functions $\mathcal{F}$ with bounded norm $\|w\| \le L_f$, the covering number of the function space is controlled by the covering number of the domain $\mathcal{M}$ via a duality argument.
Specifically, distinguishing two linear functions on $\mathcal{M}$ is equivalent to covering the domain at a finer scale $\varepsilon/L_f$. Consequently, we have:
\begin{align}
  \mathcal{N}(\mathcal{F},\varepsilon)
  &\le
  \mathcal{N}
  \Bigl(
    \mathcal{M},
    \tfrac{\varepsilon}{L_f}
  \Bigr)
  \\
  &\le
  C_0
  \left(
    \frac{D L_f}{\varepsilon}
  \right)^{d_{\mathrm{eff}}} .
\end{align}
Taking logarithms yields
\begin{equation}
  \log \mathcal{N}(\mathcal{F},\varepsilon)
  \;\le\;
  \log C_0
  +
  d_{\mathrm{eff}}\,
  \log
  \Bigl(
    \frac{D L_f}{\varepsilon}
  \Bigr) .
\end{equation}

Let $F = L_f D$ be the uniform bound on $|f(v)|$. Dudley’s entropy integral gives
\begin{equation}
  \mathcal{R}_m(\mathcal{F})
  \;\le\;
  \frac{12}{\sqrt{m}}
  \int_{0}^{F}
    \sqrt{
      \log \mathcal{N}(\mathcal{F},\varepsilon)
    }
  \,d\varepsilon .
\end{equation}
Substituting the covering number bound and simplifying (absorbing $\log C_0$ into the constant factor for asymptotic behavior), we obtain:
\begin{align}
  \mathcal{R}_m(\mathcal{F})
  &\le
  \frac{12 L_f}{\sqrt{m}}
  \int_{0}^{D}
    \sqrt{
      d_{\mathrm{eff}}
      \log
      \Bigl(
        \frac{D}{t}
      \Bigr)
    }
  \,dt ,
\end{align}
where we utilized the substitution $t = \varepsilon / L_f$.
Let $I = \int_{0}^{D} \sqrt{\log ( D/t )} \,dt$.
Using the change of variables $u = \log(D/t)$, we have $t = D e^{-u}$, which yields $I = D \int_{0}^{\infty} u^{1/2} e^{-u} du = D \cdot \Gamma(3/2) = D \frac{\sqrt{\pi}}{2}$.
Substituting this back yields:
\begin{equation}
  \mathcal{R}_m(\mathcal{F})
  \;\le\;
  \frac{12 L_f \sqrt{d_{\mathrm{eff}}}}{\sqrt{m}} \left( \frac{D \sqrt{\pi}}{2} \right)
  \;=\;
  \frac{
    6\sqrt{\pi}\,
    L_f\,D\,
    \sqrt{d_{\mathrm{eff}}}
  }{
    \sqrt{m}
  }.
\end{equation}
By setting $C_{\mathrm{Rad}} = 6\sqrt{\pi}$, we recover the bound in \eqref{eq:manifold-rad-full}.
\end{proof}

\subsection{Architectural stability of the flow}

We now turn to the flow $T_\theta$.
Let $J_\ell(v)$ be the Jacobian of the $\ell$-th invertible block at input $v$, and let $J_{\mathrm{tot}}(v)$ denote the input–output Jacobian of $T_\theta$.
The Lipschitz constant is defined as
\begin{equation}
  \operatorname{Lip}(T_\theta)
  =
  \sup_{v}
  \bigl\|
    J_{\mathrm{tot}}(v)
  \bigr\|_2 .
\end{equation}

\begin{lemma}[Architectural stability via Zero-Initialization]
\label{lem:arch-stab-full}
Consider the flow $T_\theta$ composed of $L$ blocks. Under the Zero-Initialization strategy, the Jacobian of the $\ell$-th block takes the form $J_\ell = I + \Delta_\ell$.
Let $\bar{\sigma} = \frac{1}{L} \sum_{\ell=1}^L \mathbb{E}[\|\Delta_\ell\|_2]$ be the mean spectral norm of the perturbations.
Consistent with the conservative estimation used in our implementation, the expected Lipschitz constant of the flow is bounded by:
\begin{equation}
  \mathbb{E}
  \bigl[
    \operatorname{Lip}(T_\theta)
  \bigr]
  \;\le\;
  1 + L \cdot \bar{\sigma} .
  \label{eq:lip-bound-full}
\end{equation}
\end{lemma}

\begin{proof}
The total Jacobian $J_{\mathrm{tot}}$ is the product of layer-wise Jacobians:
\begin{equation}
    J_{\mathrm{tot}} = \prod_{\ell=1}^L J_\ell = \prod_{\ell=1}^L (I + \Delta_\ell).
\end{equation}
We aim to bound the Lipschitz constant $\operatorname{Lip}(T_\theta) = \|J_{\mathrm{tot}}\|_2$. 
Recall that $J_{\mathrm{tot}} = \prod_{\ell=1}^L (I + \Delta_\ell)$. Performing a first-order expansion of this product yields:
\begin{equation}
    J_{\mathrm{tot}} \approx I + \sum_{\ell=1}^L \Delta_\ell.
\end{equation}
Applying the triangle inequality separates the identity transformation from the perturbations:
\begin{equation}
    \|J_{\mathrm{tot}}\|_2 \approx \left\| I + \sum_{\ell=1}^L \Delta_\ell \right\|_2 \le \|I\|_2 + \sum_{\ell=1}^L \|\Delta_\ell\|_2.
\end{equation}
Noting that $\|I\|_2 = 1$, we take the expectation:
\begin{align}
    \mathbb{E}[\|J_{\mathrm{tot}}\|_2] 
    &\le 1 + \sum_{\ell=1}^L \mathbb{E}[\|\Delta_\ell\|_2] \\
    &= 1 + L \cdot \left( \frac{1}{L} \sum_{\ell=1}^L \mathbb{E}[\|\Delta_\ell\|_2] \right).
\end{align}
Substituting the definition of the mean spectral norm $\bar{\sigma}$, we directly obtain:
\begin{equation}
    \mathbb{E}[\operatorname{Lip}(T_\theta)] \le 1 + L \cdot \bar{\sigma}.
\end{equation}
This confirms that under the small-perturbation regime enforced by zero-initialization, the Lipschitz constant grows linearly with depth $L$, aligning with the conservative bound used in our implementation.
\end{proof}

\subsection{Finite-sample generalization bound}

We now derive a finite-sample bound that matches the decomposition used in the main report. For clarity, we distinguish the training and validation samples
\begin{equation}
  S_{\mathrm{train}}
  =
  \{v^{\mathrm{train}}_1,\dots,v^{\mathrm{train}}_{m}\},
  \qquad
  S_{\mathrm{val}}
  =
  \{v^{\mathrm{val}}_1,\dots,v^{\mathrm{val}}_{m_{\mathrm{val}}}\},
\end{equation}
and define the corresponding empirical risks
\begin{equation}
  \hat L_{\mathrm{train}}(\theta)
  =
  \frac{1}{m}
  \sum_{i=1}^{m}
  \ell_\theta\bigl(v^{\mathrm{train}}_i\bigr).
\end{equation}

\begin{equation}
  \hat L_{\mathrm{val}}(\theta)
  =
  \frac{1}{m_{\mathrm{val}}}
  \sum_{j=1}^{m_{\mathrm{val}}}
  \ell_\theta\bigl(v^{\mathrm{val}}_j\bigr).
\end{equation}
The loss class is
\begin{equation}
  \mathcal{L}
  =
  \bigl\{
    \ell_\theta(\cdot)
    =
    -\log p_\theta(\cdot)
    :
    \theta \in \Theta
  \bigr\} .
\end{equation}
By Assumption~\ref{ass:indep-app} and Lemma~\ref{lem:arch-stab-full}, the loss functions are Lipschitz with respect to $v$. 
Using the architectural bound derived in Lemma~\ref{lem:arch-stab-full} (which establishes linear growth with depth $L$), the Rademacher complexity of $\mathcal{L}$ admits the bound:
\begin{equation}
  \mathcal{R}_m(\mathcal{L})
  \;\le\;
  \frac{
    \tilde C_{\mathrm{Rad}}\,
    L\,\bar\sigma\,
    \sqrt{d_{\mathrm{eff}}}
  }{
    \sqrt{m}
  } ,
  \label{eq:rad-loss-full}
\end{equation}
where $\tilde C_{\mathrm{Rad}}$ collects the constants from Lemma~\ref{lem:manifold-rad-full} and the affine identity term. Note that this scales linearly with $L$, consistent with our implementation.

We now control the difference between training and validation risks.
Introduce the population risk
\begin{equation}
  L(\theta)
  =
  \mathbb{E}_{v\sim\mathcal{D}}
  \bigl[
    \ell_\theta(v)
  \bigr] .
\end{equation}
By the triangle inequality,
\begin{equation}
\label{eq:train-val-triangle}
\begin{aligned}
  \bigl|
    \hat L_{\mathrm{val}}(\theta)
    -
    \hat L_{\mathrm{train}}(\theta)
  \bigr|
  &\le
  \bigl|
    L(\theta)
    -
    \hat L_{\mathrm{train}}(\theta)
  \bigr|
  \\
  &\quad
  +
  \bigl|
    \hat L_{\mathrm{val}}(\theta)
    -
    L(\theta)
  \bigr| .
\end{aligned}
\end{equation}
The two terms on the right-hand side are treated separately.

\paragraph{Training to population.}
A standard Rademacher complexity bound (see for example \citep{bartlett2002rademacher}) states that if $|\ell_\theta(v)| \le M_{\mathrm{train}}$ for all $\theta \in \Theta$ and all $v$ in the support of the training distribution, then for any $\delta \in (0,1)$, with probability at least $1-\delta$ over the draw of $S_{\mathrm{train}}$ one has
\begin{equation}
\label{eq:bm-train}
\begin{aligned}
  \bigl|
    L(\theta)
    -
    \hat L_{\mathrm{train}}(\theta)
  \bigr|
  &\le
  2 \mathcal{R}_m(\mathcal{L})
  \\
  &\quad
  +
  3 M_{\mathrm{train}}
  \sqrt{
    \frac{
      \log(2/\delta)
    }{
      2 m
    }
  } .
\end{aligned}
\end{equation}
Combining \eqref{eq:bm-train} with \eqref{eq:rad-loss-full} gives
\begin{equation}
\label{eq:train-pop-bound}
\begin{aligned}
  \bigl|
    L(\theta)
    -
    \hat L_{\mathrm{train}}(\theta)
  \bigr|
  &\le
  \frac{
    2 \tilde C_{\mathrm{Rad}}\,
    L\,\bar\sigma\,
    \sqrt{d_{\mathrm{eff}}}
  }{
    \sqrt{m}
  }
  \\
  &\quad
  +
  3 M_{\mathrm{train}}
  \sqrt{
    \frac{
      \log(2/\delta)
    }{
      2 m
    }
  } .
\end{aligned}
\end{equation}

\paragraph{Validation to population.}
For the validation set, the model parameter $\theta$ is fixed, so we only need a concentration bound for bounded random variables.
Let $X_j = \ell_\theta(v^{\mathrm{val}}_j)$ with $X_j \in [a,b]$ and define $M_{\mathrm{val}} = b-a$.
Hoeffding’s inequality \citep{Hoeffding1994} yields
\begin{equation}
  \mathbb{P}
  \biggl(
    \Bigl|
      \hat L_{\mathrm{val}}(\theta)
      -
      L(\theta)
    \Bigr|
    \ge t
  \biggr)
  \le
  2 \exp
  \biggl(
    -
    \frac{
      2 m_{\mathrm{val}} t^2
    }{
      M_{\mathrm{val}}^2
    }
  \biggr) .
\end{equation}
Setting the right-hand side to $\delta$ and solving for $t$ gives that, with probability at least $1-\delta$,
\begin{equation}
\label{eq:val-pop-bound}
  \bigl|
    \hat L_{\mathrm{val}}(\theta)
    -
    L(\theta)
  \bigr|
  \;\le\;
  M_{\mathrm{val}}
  \sqrt{
    \frac{
      \log(2/\delta)
    }{
      2 m_{\mathrm{val}}
    }
  } .
\end{equation}

\paragraph{Final bound.}
Combining \eqref{eq:train-val-triangle}, \eqref{eq:train-pop-bound} and \eqref{eq:val-pop-bound}, and applying a union bound, we obtain the following result.

\setcounter{theorem}{0}

\begin{theorem}[Finite-sample generalization bound]
\label{thm:gen-bound-full}
Under Assumptions~\ref{ass:manifold-app} and \ref{ass:indep-app}, for any fixed flow model $p_\theta$ and any $\delta \in (0,1)$, with probability at least $1-\delta$ over the draws of the training set of size $m$ and the validation set of size $m_{\mathrm{val}}$, the train–validation gap satisfies
\begin{equation}
\label{eq:gen-bound-main-app-restate}
\begin{aligned}
  \bigl|
    \hat L_{\mathrm{val}}(\theta)
    -
    \hat L_{\mathrm{train}}(\theta)
  \bigr|
  &\le
  \frac{
    2 \tilde C_{\mathrm{Rad}}\,
    L\,\bar\sigma\,
    \sqrt{d_{\mathrm{eff}}}
  }{
    \sqrt{m}
  }
  \\
  &\quad
  +
  M_{\mathrm{val}}
  \sqrt{
    \frac{
      \log(2/\delta)
    }{
      2 m_{\mathrm{val}}
    }
  }
  \\
  &\quad
  +
  3 M_{\mathrm{train}}
  \sqrt{
    \frac{
      \log(2/\delta)
    }{
      2 m
    }
  } .
\end{aligned}
\end{equation}
\end{theorem}
\subsection{Extension to the IS score}

Let $L_{\mathrm{marg}}$ and $L_{\mathrm{cond}}$ denote the population loss of the marginal and conditional flows, and let $\hat L_{\mathrm{marg}}$ and $\hat L_{\mathrm{cond}}$ be the corresponding validation losses.
The population $I_s$ score is
\begin{equation}
    IS = L_{\mathrm{marg}}(V) - L_{\mathrm{cond}}(V|U),
\end{equation}
and the empirical $I_s$ score is
\begin{equation}
    \widehat{IS} = \hat L_{\mathrm{marg}}(V) - \hat L_{\mathrm{cond}}(V|U).
\end{equation}
The difference is
\begin{equation}
    IS - \widehat{IS} = (L_{\mathrm{marg}} - \hat L_{\mathrm{marg}}) - (L_{\mathrm{cond}} - \hat L_{\mathrm{cond}}).
\end{equation}
By the triangle inequality,
\begin{equation}
    |IS - \widehat{IS}| \le |L_{\mathrm{marg}} - \hat L_{\mathrm{marg}}| + |L_{\mathrm{cond}} - \hat L_{\mathrm{cond}}|.
\end{equation}
Applying Theorem~\ref{thm:gen-bound-full} (Eq.~\ref{eq:gen-bound-main-app-restate}) separately to the marginal and conditional flows yields a finite-sample upper bound for the generalization error of the $I_s$ estimator as the sum of the two individual generalization gaps.

\section{Experiment Detail}
\label{app:detail}
In this section, we provide the experimental details necessary to reproduce these experiments.

\subsection{Evaluated Model and Datasets Details}
In Table~\ref{tab:metadatafomodel}, we provide the metadata of the evaluated models and their score on the 11 datasets. We provide the statistics of the datasets used to evaluate $\bar{I}_s$ in table~\ref{tab:dataset_stats}.

\begin{table}[htbp]
  \centering
    \begin{tabular}{lcc}
    \toprule
    Model & Dim.  & $\bar{I}_S$ \\
    \midrule
    Zeta\_Alpha\_E5\_Mistral & 4096  & \textbf{0.20 } \\
    Linq\_Embed\_Mistral & 4096  & \textbf{0.20 } \\
    SFR\_Embedding\_Mistral & 4096  & \textbf{0.19 } \\
    bge\_multilingual\_gemma2 & 3584  & \textbf{0.19 } \\
    GritLM\_7B & 4096  & \textbf{0.18 } \\
    gte\_Qwen2\_7B\_instruct & 3584  & \textbf{0.17 } \\
    stella\_base\_en\_v2 & 768   & \textbf{0.13 } \\
    all\_MiniLM\_L6\_v2 & 384   & \textbf{0.13 } \\
    \bottomrule
    \end{tabular}%
    \caption{Information sufficiency of the evaluated models by \methodname\ .}
  \label{tab:metadatafomodel}%
\end{table}%

\subsection{Downstream Task Evaluation}
We deliberately select datasets that are either newly released or underexplored to minimize the risk of data leakage during embedding model pretraining. As no established performance benchmarks for embedding models exist on these datasets, we evaluate downstream task performance ourselves following the MTEB evaluation ~\citep{muennighoff-etal-2023-mteb} protocol: F1 Macro for classification, Spearman for STS, nDCG@10 for retrieval, and V-measure for clustering.

\begin{table*}[h]
  \centering
  \begin{tabular}{llrrr}
    \toprule
    \textbf{Dataset} & \textbf{Task} & \textbf{Train} & \textbf{Val} & \textbf{Total} \\
    \midrule
    apt-eval & Classification & 13,185 & 1,465 & 14,650 \\
    gtfintechlab & Classification & 12,250 & 2,625 & 14,875 \\
    BhashaBench-Finance & Classification & 12,105 & 1,346 & 13,451 \\
    \midrule
    Aug-STSB & STS & 33,635 & 3,738 & 37,373 \\
    LivNLP-STS & STS & 12,758 & 1,418 & 14,176 \\
    Philosophical-STS & STS & 58,560 & 6,507 & 65,067 \\
    \midrule
    AIR-Bench & Retrieval & 23,639 & 2,627 & 26,266 \\
    LIMIT & Retrieval & 45,000 & 5,000 & 50,000 \\
    arXiv '25 & Retrieval & 2,610 & 290 & 2,900 \\
    \midrule
    FunPang & Clustering & 28,716 & 3,191 & 31,907 \\
    Reasoning & Clustering & 50,772 & 5,642 & 56,414 \\
    \bottomrule
  \end{tabular}
  \caption{Statistics of the datasets used in our experiments.}
  \label{tab:dataset_stats}
\end{table*}

Table~\ref{tab:downstream} reports the average downstream task performance of the eight evaluated embedding models, aggregated by task type.
Overall, 7B-scale instruction-tuned models (GritLM, SFR, Linq, Zeta-Alpha) consistently outperform smaller models across all task categories.
GritLM-7B achieves the best classification performance (0.61), while SFR-Embedding-Mistral leads on STS (0.67).
The two smaller models, stella-base-en-v2 and all-MiniLM-L6-v2, show competitive performance on classification but lag significantly on retrieval and STS tasks.

Retrieval scores exhibit the largest variance, partly because LIMIT is designed to stress-test embedding-model capacity and yields low absolute scores for all candidates.

\begin{table*}[htbp]
  \centering
    \begin{tabular}{lcccc}
    \toprule
    Model & Classification & STS   & Retrieval & Clustering \\
    \midrule
    bge\_multilingual\_gemma2 & 0.58  & 0.63  & 0.46  & 0.30  \\
    Zeta\_Alpha\_E5\_Mistral & 0.57  & 0.65  & 0.53  & 0.31  \\
    GritLM\_7B & 0.61  & 0.64  & 0.50  & 0.32  \\
    SFR\_Embedding\_Mistral & 0.57  & 0.67  & 0.53  & 0.26  \\
    Linq\_Embed\_Mistral & 0.60  & 0.66  & 0.52  & 0.32  \\
    gte\_Qwen2\_7B\_instruct & 0.57  & 0.64  & 0.50  & 0.28  \\
    stella\_base\_en\_v2 & 0.49  & 0.45  & 0.05  & 0.28  \\
    all\_MiniLM\_L6\_v2 & 0.50  & 0.51  & 0.47  & 0.21  \\
    \bottomrule
    \end{tabular}%
    \caption{Summary of the evaluated embedders and their performance on downstream datasets.}
  \label{tab:downstream}%
\end{table*}%

\subsection{Training Configuration}
\label{app:training_config}

The flow architecture itself is described in Section~\ref{sec:method-flow}. Here we list the optimization hyperparameters, which are shared across all 11 datasets and all $N{=}8$ embedders without any per-dataset tuning, so that the same training recipe is used end-to-end. Both the marginal flow $p_\phi(v)$ and the conditional flow $p_\theta(v\mid u)$ are trained with AdamW under mixed precision (AMP), weight decay $1{\times}10^{-3}$, and EMA decay $0.999$. The marginal flow uses initial learning rate $2{\times}10^{-2}$, batch size $256$, gradient accumulation $2$, and at most $1{,}000$ epochs; the conditional flow uses initial learning rate $1{\times}10^{-1}$, batch size $64$, gradient accumulation $4$, and at most $500$ epochs.

\subsection{Comprehensive Results}

The primary goal of our method is to \emph{rank} candidate embedding models so that practitioners can select the best one for a given downstream task without access to labeled data.
The correlation coefficients reported in Table~\ref{tab:main_results} validate that the rankings produced by \methodname\ align with ground-truth task performance.

We compare \methodname\ against four unsupervised baselines: Uniformity, IsoScore, Silhouette Score, and EMIR.
Across 11 datasets spanning classification, STS, retrieval, and clustering, \methodname\ attains the highest average Spearman correlation ($\rho = 0.70$) and is the only method whose correlation is positive on every dataset.

EMIR achieves an average Spearman correlation of only $-0.12$, indicating that its rankings frequently contradict ground-truth performance. This is consistent with the well-known degradation of kernel-based density estimators in high-dimensional spaces.
On individual datasets, EMIR shows high variance: while achieving moderate positive correlations on some tasks (e.g., $\rho = 0.33$ on LivNLP-STS), it produces strongly negative correlations on others (e.g., $\rho = -0.52$ on FunPang, $\rho = -0.43$ on arXiv '25).
This instability limits its practical utility for model selection.

IsoScore exhibits consistently poor performance with an average of $\rho = -0.27$.
These results confirm that \methodname\ provides the most reliable unsupervised signal for embedding model selection.

\begin{table*}[t]
\centering
\renewcommand{\arraystretch}{1.2}
\resizebox{\textwidth}{!}{%
    \begin{tabular}{ll | ccccc | ccccc}
    \toprule
    & & \multicolumn{5}{c|}{\textbf{Spearman's $\rho$}} & \multicolumn{5}{c}{\textbf{Pearson's $r$}} \\
    \cmidrule(lr){3-7} \cmidrule(lr){8-12}
    \textbf{Dataset} & \textbf{Task} & \textbf{Uni.} & \textbf{Iso.} & \textbf{Sil.} & \textbf{EMIR} & \textbf{Ours} & \textbf{Uni.} & \textbf{Iso.} & \textbf{Sil.} & \textbf{EMIR} & \textbf{Ours} \\
    \midrule
    apt-eval & Class. & 0.36 & -0.24 & -0.21 & 0.07 & \textbf{0.43} & 0.50 & -0.70 & -0.42 & -0.01 & 0.20 \\
    gtfintechlab & Class. & -0.12 & -0.50 & 0.00 & -0.38 & 0.43 & 0.11 & -0.42 & -0.18 & -0.46 & 0.14 \\
    BhashaBench & Class. & 0.29 & -0.45 & 0.02 & 0.12 & 0.81 & 0.44 & -0.62 & 0.35 & -0.12 & 0.59 \\
    \midrule
    Aug-STSB & STS & 0.14 & 0.05 & 0.24 & -0.12 & \textbf{0.83} & 0.51 & 0.25 & 0.06 & -0.11 & \textbf{0.68} \\
    LivNLP-STS & STS & -0.18 & -0.27 & 0.81 & 0.33 & \textbf{0.83} & 0.77 & -0.41 & 0.83 & 0.39 & \textbf{0.94} \\
    Philo-STS & STS & 0.07 & -0.76 & \textbf{0.62} & -0.38 & 0.43 & \textbf{0.82} & -0.26 & 0.70 & -0.51 & 0.49 \\
    \midrule
    AIR-Bench & Retr. & -0.14 & -0.79 & 0.05 & -0.24 & \textbf{0.71} & \textbf{0.70} & -0.27 & 0.45 & 0.07 & 0.69 \\
    arXiv '25 & Retr. & 0.55 & 0.14 & -0.71 & -0.43 & \textbf{0.76} & \textbf{0.72} & -0.27 & 0.10 & -0.33 & 0.62 \\
    LIMIT & Retr. & -0.17 & -0.71 & 0.57 & 0.02 & \textbf{0.81} & -0.14 & -0.57 & \textbf{0.62} & -0.28 & \textbf{0.62} \\
    \midrule
    FunPang & Clust. & -0.14 & 0.50 & -0.48 & -0.52 & \textbf{0.90} & -0.52 & 0.82 & -0.75 & -0.57 & \textbf{0.83} \\
    Reasoning & Clust. & -0.14 & 0.07 & 0.00 & 0.21 & \textbf{0.76} & -0.33 & -0.73 & 0.26 & 0.37 & \textbf{0.55} \\
    \midrule
    \multicolumn{2}{l|}{\textbf{Average}} & 0.05 & -0.27 & 0.08 & -0.12 & \textbf{0.70} & 0.33 & -0.29 & 0.18 & -0.14 & \textbf{0.58} \\
    \bottomrule
    \end{tabular}%
}
\caption{Comparisons with unsupervised baselines. \methodname\ achieves the highest consistency and average correlation.}
\label{tab:main_results}
\end{table*}

\section{Ranking Stability under Subsampling}
\label{app:ranking-stability}

We evaluate whether the proposed \methodname\ yields stable model rankings when the evaluation set is randomly subsampled.
For each dataset, we subsample the evaluation set at ratios
$\alpha \in \{5\%, 10\%, 20\%, 40\%, 60\%, 80\%, 100\%\}$
without replacement, and repeat this process 20 times.
For each $\alpha$, we recompute scores using the same pretrained marginal and conditional flows and obtain a ranking of embedding models.
We then compute the Spearman rank correlation $\rho(\alpha)$ between the ranking induced by the subsampled evaluation set and the reference ranking computed on the full evaluation set ($\alpha=1.0$), and report the deviation
\begin{equation}
\Delta_{\rho}(\alpha) \;=\; \left| \rho(\alpha) - \rho(1.0) \right|.
\end{equation}

We focus on Spearman correlation because our goal is to assess the stability of \emph{model ranking} for model selection, rather than the linear agreement of raw scores.
Rank-based measures directly quantify whether the relative ordering of models is preserved under subsampling, which is the main quantity of interest in this analysis.

Figure~\ref{fig:sensitivity_by_task} shows the ranking deviation $\Delta_{\rho}(\alpha) = |\rho_{\alpha} - \rho_{\text{full}}|$ as a function of the subsampling ratio $\alpha$ across all 11 datasets, grouped by task type.
Overall, $I_s$ rankings remain highly stable under subsampling: for 8 out of 11 datasets, $\Delta_{\rho} < 0.05$ even when using only 20\% of the validation data.

Among task types, \textbf{clustering} exhibits the highest stability, with both datasets maintaining $\Delta_{\rho} < 0.025$ across all subsampling ratios.
\textbf{Classification} and \textbf{retrieval} tasks also demonstrate strong robustness, with most datasets showing only minor deviations ($\Delta_{\rho} < 0.045$) even at very low $\alpha$.

In contrast, \textbf{STS} datasets display larger variance at small sample sizes; \textit{Aug-STSB} shows the highest deviation ($\Delta_{\rho} \approx 0.12$) at $\alpha = 0.05$, consistent with correlation-based evaluation being more sample-sensitive.
However, these deviations diminish rapidly: once $\alpha \geq 0.2$, all STS datasets achieve $\Delta_{\rho} < 0.07$.

These results demonstrate that $I_s$ selection is robust to evaluation subsampling, with 20--40\% of validation data typically sufficient for reliable model ranking across diverse task types.

\begin{figure*}[htbp]
\centering
\includegraphics[width=\linewidth]{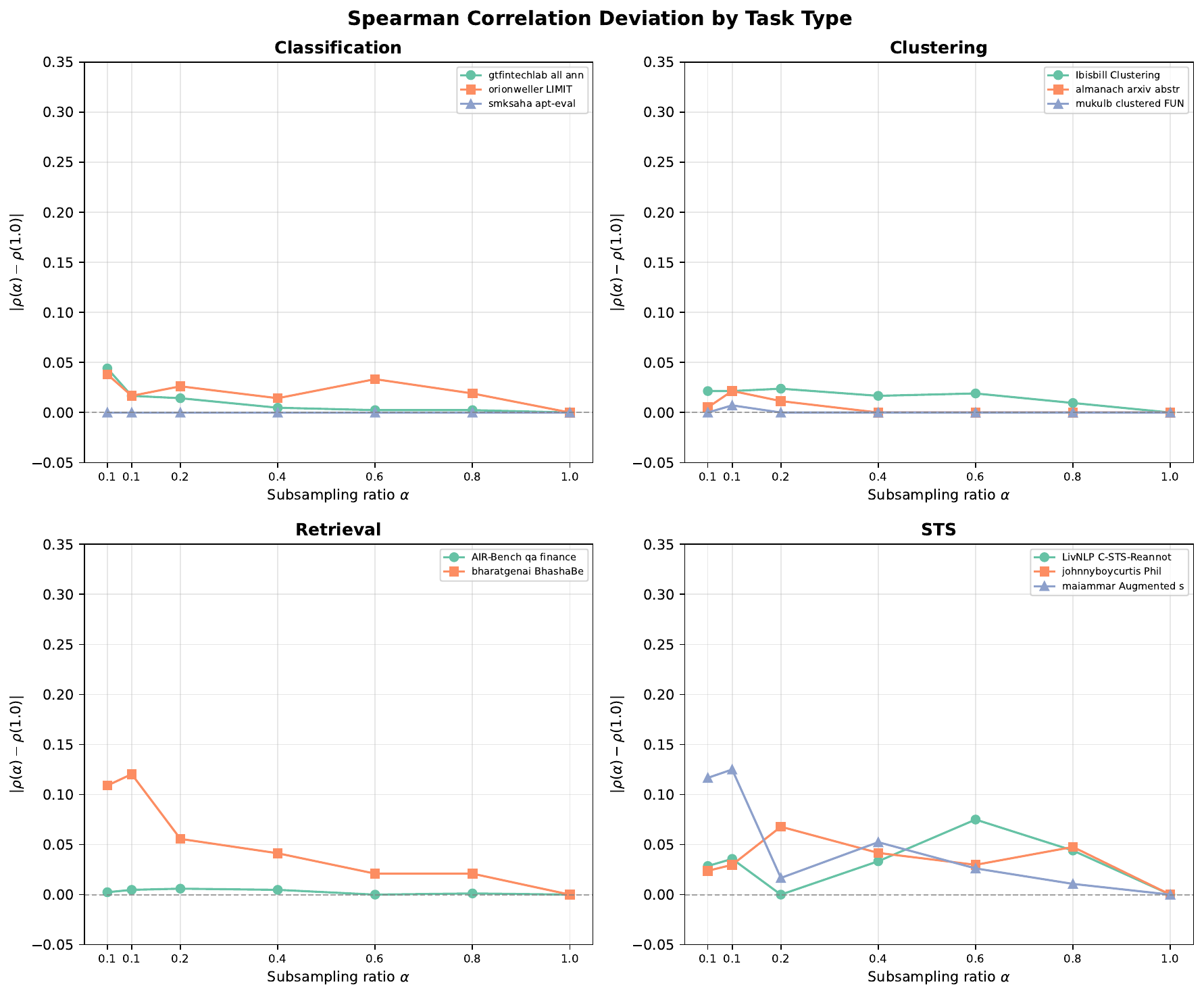}
\caption{\textbf{Ranking stability under evaluation subsampling.}
  We report the deviation $\Delta_{\rho}(\alpha)=|\rho(\alpha)-\rho(1.0)|$, where $\rho(\alpha)$ is the Spearman rank correlation between the model ranking induced by IS scores computed on a subsampled evaluation set (ratio $\alpha$) and the ranking computed on the full evaluation set ($\alpha=1.0$).
  Smaller values indicate more stable rankings.}
\label{fig:sensitivity_by_task}
\end{figure*}

\section{Ablation Study}
\label{app:ablation}

In this section, we conduct a set of ablation studies to better understand the factors that contribute to the effectiveness of our method.
Unless otherwise specified, all experiments are conducted using the same evaluation protocol and datasets as in the main experiments.

\paragraph{Experimental setup.}
For all ablation experiments, we use the same pretrained models and evaluation datasets as in the main experiments.
Unless otherwise stated, we recompute the evaluation scores under modified settings while keeping all other components unchanged.
Performance is measured using Spearman correlation between the predicted ranking and the ground-truth ranking.

\subsection{Shuffle Ablation}
\label{sec:shuffle}
To examine whether the proposed metric truly relies on the correspondence between paired representations, we conduct a shuffle-based ablation with varying shuffle ratios.
Specifically, for a given ratio $p \in \{0, 0.1, 0.2, 0.4, 0.6, 0.8, 1.0\}$, we randomly select a $p$ fraction of the evaluation samples and permute the correspondence between $U$ and $V$ within this subset, while keeping the remaining $(1-p)$ portion unchanged.
This procedure preserves the marginal distributions of both $U$ and $V$ but progressively destroys their pairwise alignment as $p$ increases.

For each shuffle ratio, we recompute the proposed score and evaluate the resulting ranking against the ground-truth downstream performance.
As shown in Table~\ref{tab:shuffle_ablation}, the Spearman correlation generally degrades as $p$ increases across datasets.
At $p=0$, all datasets exhibit positive correlations (avg.\ $\rho = 0.70$); at $p=1.0$, correlations become predominantly negative (avg.\ $\rho \approx -0.40$).
The transition point varies by task type: Classification and STS datasets tend to flip sign at lower shuffle ratios ($p \approx 0.1$), while Retrieval and Clustering datasets maintain positive correlations longer (up to $p \approx 0.2\text{--}0.4$).
This behavior confirms that the proposed metric critically depends on correct $U$-$V$ alignment rather than marginal statistics alone.
\begin{table*}[htbp]
\centering
\begin{tabular}{lcccccccc}
\toprule
\textbf{Dataset} & \textbf{Task} & \textbf{$p=0$} & \textbf{$p=0.1$} & \textbf{$p=0.2$} & \textbf{$p=0.4$} & \textbf{$p=0.6$} & \textbf{$p=0.8$} & \textbf{$p=1.0$} \\
\midrule
apt-eval     & Cls & 0.43 & $-$0.36 & 0.48 & $-$0.48 & $-$0.43 & $-$0.48 & $-$0.48 \\
gtfintechlab & Cls & 0.43 & $-$0.19 & $-$0.24 & $-$0.62 & $-$0.69 & $-$0.74 & $-$0.64 \\
BhashaBench  & Cls & 0.81 & $-$0.74 & $-$0.62 & $-$0.50 & $-$0.50 & $-$0.43 & $-$0.26 \\ 
\midrule
Aug-STSB     & STS & 0.83 & $-$0.83 & $-$0.81 & $-$0.81 & $-$0.93 & $-$0.57 & $-$0.71 \\
LivNLP-STS   & STS & 0.83 & $-$0.07 & $-$0.10 & $-$0.55 & $-$0.52 & $-$0.52 & $-$0.43 \\
Philo-STS    & STS & 0.43 & $-$0.57 & $-$0.48 & $-$0.45 & $-$0.40 & $-$0.36 & $-$0.24 \\
\midrule
AIR-Bench    & Ret & 0.71 & 0.76 & 0.24 & 0.69 & $-$0.02 & 0.02 & $-$0.55 \\
arXiv '25    & Ret & 0.76 & 0.71 & 0.36 & 0.64 & 0.43 & 0.50 & 0.48 \\ 
LIMIT        & Ret & 0.81 & 0.31 & 0.29 & $-$0.81 & $-$0.67 & $-$0.36 & $-$0.48 \\
\midrule
FunPang      & Clust & 0.90 & 0.93 & 0.67 & $-$0.83 & $-$0.90 & $-$0.90 & $-$0.90 \\
Reasoning    & Clust & 0.76 & 0.26 & $-$0.43 & $-$0.07 & $-$0.02 & $-$0.21 & $-$0.24 \\ 
\bottomrule
\end{tabular}
\caption{Per-dataset Spearman correlation ($\rho$) under partial shuffle ablation. The column $p=0$ corresponds to the full method without shuffling. As shuffle proportion $p$ increases, correlation with downstream performance generally degrades.}
\label{tab:shuffle_ablation}
\end{table*}

\subsection{Full $I_s$ vs.\ Conditional-Only}
\label{sec:ablation_marginal}

To investigate the contribution of the marginal term in the $I_s$, we conduct an ablation study comparing the full $I_s$ formulation against a conditional-only variant that omits the marginal likelihood component.

\paragraph{Results.}
Table~\ref{tab:ablation_cond_only} and Figure~\ref{fig:cond_only} present the comparison between Full $I_s$ and Cond Only across all 11 datasets.
The full $I_s$ formulation achieves an average Spearman correlation of $\rho = 0.70$, substantially outperforming the Cond Only variant ($\rho = 0.21$). 
Notably, Full $I_s$ outperforms Cond Only on 9 out of 11 datasets. 
The advantage is particularly pronounced on retrieval and clustering tasks: on LIMIT, Full $I_s$ achieves $\rho = 0.81$ compared to $-0.21$ for Cond Only; on FunPang, the gap is $0.90$ vs.\ $-0.36$.
These results indicate that the conditional term alone captures only partial information about embedding quality. It measures how well the source embedding predicts the target, but fails to account for the intrinsic structure of the target embedding space.

Interestingly, on a few datasets (apt-eval, gtfintechlab), Cond Only slightly outperforms Full $I_s$. 
However, its performance is highly inconsistent, with five datasets showing negative correlations.
Full $I_s$, in contrast, stays positive on all 11 datasets.

These findings confirm that both terms in the $I_s$ formulation are necessary: the marginal term captures target-side embedding quality, while the conditional term measures cross-model information transfer.
Their combination yields a more reliable and consistent signal for unsupervised model selection.

\begin{figure*}[htbp]
\centering
\includegraphics[width=\linewidth]{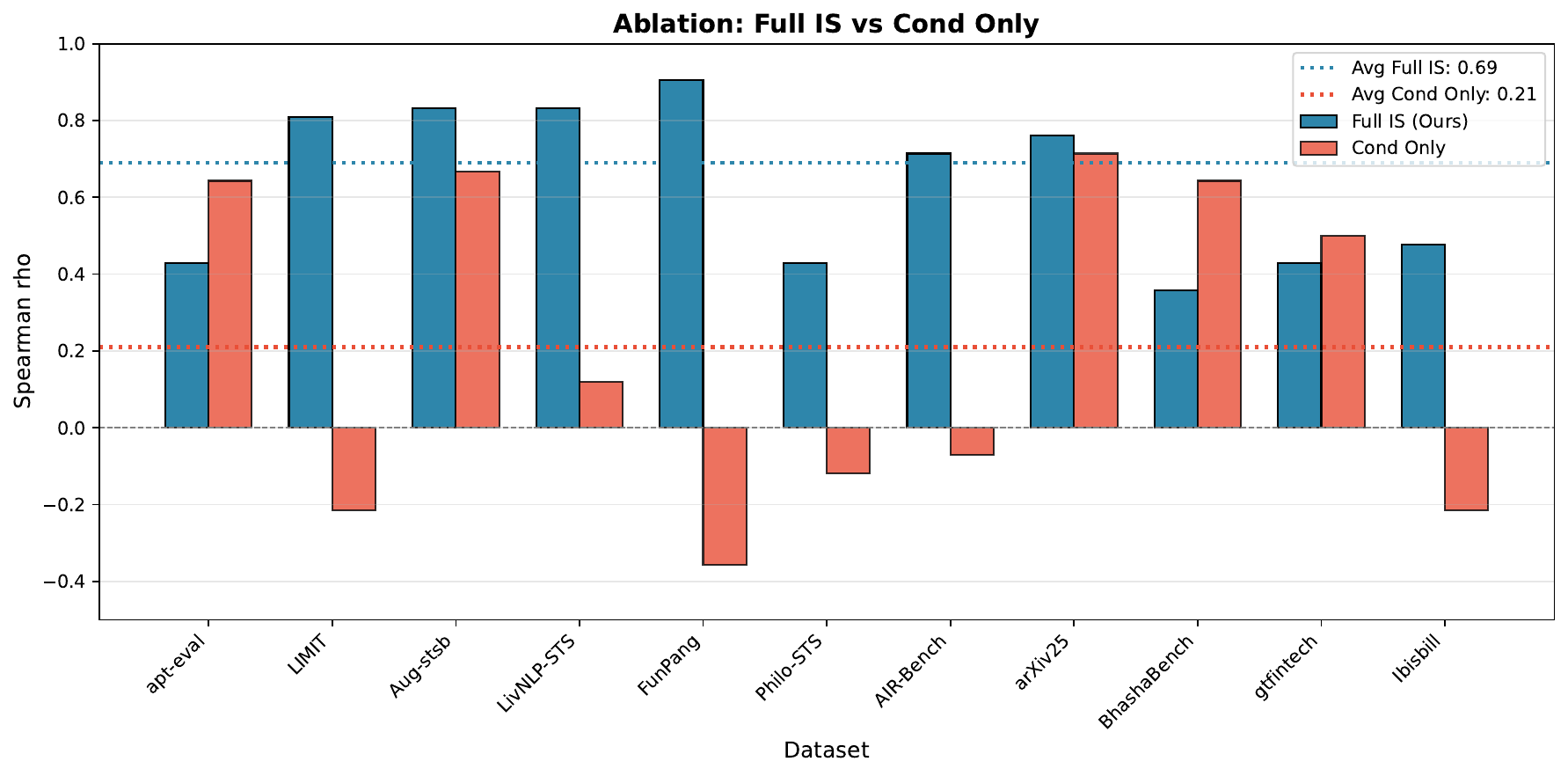}
\caption{Ablation: Comparison of correlation of Full $I_s$ vs.\ conditional-only component.}
\label{fig:cond_only}
\end{figure*}

\begin{table}[htbp]
\centering
\begin{tabular}{lcc}
\toprule
\textbf{Dataset} & \textbf{Full IS} & \textbf{Cond Only} \\
\midrule
apt-eval        & 0.43 & \textbf{0.64} \\
LIMIT           & \textbf{0.81} & $-$0.21 \\
Aug-STSB       & \textbf{0.83} & 0.67 \\
LivNLP-STS      & \textbf{0.83} & 0.12 \\
FunPang         & \textbf{0.90} & $-$0.36 \\
Philo-STS       & \textbf{0.43} & $-$0.12 \\
AIR-Bench       & \textbf{0.71} & $-$0.07 \\
arXiv '25       & \textbf{0.76} & 0.71 \\ 
BhashaBench     & \textbf{0.81} & 0.64 \\ 
gtfintechlab    & 0.43 & \textbf{0.50} \\ 
Reasoning       & \textbf{0.76} & $-$0.21 \\
\midrule
\textbf{Average} & \textbf{0.70} & 0.21 \\
\bottomrule
\end{tabular}
\caption{Ablation comparing Full $I_s$ with Cond Only. Full $I_s$ achieves substantially higher correlation with ground truth rankings (avg.\ $\rho = 0.70$ vs.\ $0.21$), confirming the importance of the marginal term.}
\label{tab:ablation_cond_only}
\end{table}

\subsection{Aggregation Strategy}
We investigate the impact of different aggregation strategies for combining IS scores into a single model score. We compare three methods: (1) arithmetic mean, (2) median, and (3) 10\% trimmed mean (Trim10), which discards the top and bottom 10\% of values before averaging.

As shown in Table~\ref{tab:ablation_aggregation}, median aggregation consistently outperforms alternatives, achieving the highest Spearman correlation on 9 of 11 datasets (average $\rho = 0.70$ vs.\ $0.52$ for mean, a relative improvement of 34.6\%). 
For Pearson correlation, median also leads with an average of $r = 0.56$ compared to $0.54$ for mean and $0.53$ for trimmed mean. 

The advantage of median aggregation is most pronounced on STS, where it reaches $\rho = 0.83$ on both Aug-STSB and LivNLP-STS, compared to $\rho \le 0.45$ for the mean. The mean is pulled by a small number of pairs with anomalously high or low IS scores, while the median is unaffected.

Exceptions include Philo-STS and arXiv '25, where mean or trimmed mean aggregation outperforms the median. 
For Philo-STS, mean aggregation ($\rho = 0.62$) surpasses median ($\rho = 0.43$), likely due to the smaller dataset size reducing the prevalence of outlier scores. 
Similarly, on arXiv '25, trimmed mean achieves the highest correlation ($\rho = 0.84$), suggesting that while some outliers exist, the distribution tails might contain valuable signal for this specific retrieval task.
Even with these exceptions, the median is the safest default across task families.

Based on these findings, we adopt median aggregation as the default strategy throughout our experiments.

\begin{table*}[htbp]
\centering
\begin{tabular}{llcccccc}
\toprule
& & \multicolumn{3}{c}{\textbf{Spearman $\rho$}} & \multicolumn{3}{c}{\textbf{Pearson $r$}} \\
\cmidrule(lr){3-5} \cmidrule(lr){6-8}
\textbf{Dataset} & \textbf{Task} & Mean & Median & Trim10 & Mean & Median & Trim10 \\
\midrule
apt-eval & Class. & 0.29 & 0.43 & 0.29 & 0.26 & 0.20 & 0.16 \\
gtfintechlab & Class. & 0.38 & 0.43 & 0.17 & 0.36 & 0.14 & -0.05 \\
BhashaBench & Class. & 0.53 & \textbf{0.81} & 0.36 & 0.69 & 0.59 & 0.45 \\ 
\midrule
Aug-STSB & STS & 0.40 & \textbf{0.83} & 0.69 & 0.54 & 0.68 & 0.63 \\
LivNLP-STS & STS & 0.45 & \textbf{0.83} & 0.52 & 0.55 & \textbf{0.94} & 0.53 \\
Philo-STS & STS & \textbf{0.62} & 0.43 & 0.52 & \textbf{0.66} & 0.49 & 0.63 \\
\midrule
AIR-Bench & Retr. & 0.43 & \textbf{0.71} & 0.48 & 0.52 & 0.70 & \textbf{0.80} \\
arXiv '25 & Retr. & 0.70 & 0.76 & \textbf{0.84} & 0.25 & 0.46 & \textbf{0.56} \\ 
LIMIT & Retr. & 0.62 & \textbf{0.81} & \textbf{0.81} & 0.61 & 0.62 & \textbf{0.65} \\
\midrule
FunPang & Clust. & 0.81 & \textbf{0.90} & 0.88 & \textbf{0.95} & 0.83 & 0.93 \\
Reasoning & Clust. & 0.52 & \textbf{0.76} & 0.60 & \textbf{0.56} & 0.55 & \textbf{0.57} \\
\midrule
\multicolumn{2}{l}{\textbf{Average}} & 0.52 & \textbf{0.70} & 0.56 & 0.54 & \textbf{0.56} & 0.53 \\ 
\bottomrule
\end{tabular}
\caption{Aggregation ablation comparing mean, median, and trimmed mean (10\%). The median aggregation (Ours) achieves the highest consistency ($\rho=0.70$).}
\label{tab:ablation_aggregation}
\end{table*}

\section{Additional Analysis Details}
\label{app:additional}

This section provides full experimental results for the analyses summarized in Section~\ref{subsec:additional_analysis}.

\subsection{Assumption Validation}
\label{app:assumption_validation}

This section presents the full experimental details for the three validation experiments summarized in Section~\ref{subsec:additional_analysis}.

\paragraph{Logical structure of the test.}
Assumption~\ref{ass:indep-app} states that the dominant Jacobian singular directions of different coupling layers are approximately independent. Together with a per-layer Lipschitz factor close to $1$, this implies that the total Lipschitz constant of the $L$-layer flow grows approximately linearly in $L$ rather than exponentially, which is the regime under which Theorem~\ref{thm:gen-bound} provides a useful generalization bound. The three experiments are organized to verify this chain from complementary perspectives. Experiment~1 tests the core geometric assumption by checking that the dominant Jacobian directions of different layers are not aligned. Experiment~2 verifies the required local condition by showing that per-layer amplification remains close to $1$ while layers still perform non-trivial transformations. Experiment~3 examines the global consequence by directly probing whether total perturbation amplification remains $O(1)$ in trained flows. The assumption itself cannot be exhaustively verified in high-dimensional parameter spaces; instead, the quantitative agreement among these three empirical observations is what makes Assumption~\ref{ass:indep-app} plausible in practice. In terms of scope, Experiment~1 is conducted on a small set of representative conditional flows from the FUNPANG dataset, while Experiments~2 and 3 aggregate statistics across conditional flows from all 11 datasets.

\paragraph{Experiment 1: Singular direction independence.}
We first estimate each coupling layer's Jacobian dominant singular vector via randomized finite-difference probing and compute pairwise $|\cos\theta|$ across all layer pairs (Table~\ref{tab:cosine_sim}). Two independent random vectors in $\mathbb{R}^d$ have expected $|\cos\theta|\approx 1/\sqrt{d}=0.016$, which provides the natural \emph{independence baseline}: values at or below $0.016$ indicate that dominant directions of different layers are no more aligned than random vectors.

\begin{table}[htbp]
\centering
\begin{tabular}{lccc}
\toprule
\textbf{Model Pair} & \textbf{Mean $|\cos\theta|$} & \textbf{Max} & \textbf{Ratio} \\
\midrule
bge $\to$ Zeta & 0.012 & 0.040 & 0.75$\times$ \\
GritLM $\to$ SFR & 0.011 & 0.030 & 0.69$\times$ \\
gte $\to$ Linq & 0.009 & 0.024 & 0.56$\times$ \\
\midrule
\textbf{Average} & \textbf{0.010} & \textbf{0.040} & \textbf{0.65$\times$} \\
\bottomrule
\end{tabular}
\caption{Inter-layer cosine similarity of dominant singular vectors. The random baseline is $1/\sqrt{d}=0.016$; all observed values are below it.}
\label{tab:cosine_sim}
\end{table}

The mean $|\cos\theta|=0.010$ is consistently \emph{below} this baseline (ratio $0.65\times$ on average across the three pairs), and the worst single layer pair only reaches $0.040$. In other words, dominant singular directions of different coupling layers are no more aligned than randomly drawn unit vectors. We attribute this to the alternating binary masks and the random coordinate permutation applied after every coupling block in our NSF backbone (Section~\ref{sec:method-flow}): each permutation rotates the ``active'' coordinate subspace, so consecutive layers structurally cannot inherit each other's dominant directions. The three model pairs span very different source/target architectures yet produce nearly identical numbers, supporting our claim in Section~\ref{subsec:additional_analysis} that this is a property of the backbone rather than of any particular embedding.

To strengthen the test beyond a single direction, we also compute principal angles between the top-$3$ singular subspaces of adjacent layers (Table~\ref{tab:principal_angles}). Every reported angle lies in $[87.6^\circ,89.7^\circ]$: the top-$3$ subspace of one layer projects onto that of the next layer with magnitude at most $\cos(87.6^\circ)\approx 0.04$, so even when we widen the comparison from one direction to a three-dimensional subspace the layers remain essentially orthogonal.

\begin{table}[htbp]
\centering
\begin{tabular}{lc}
\toprule
\textbf{Layer Pair} & \textbf{Principal Angles} \\
\midrule
$L_0 \leftrightarrow L_1$ & $87.8^\circ,\; 88.4^\circ,\; 89.6^\circ$ \\
$L_1 \leftrightarrow L_2$ & $87.6^\circ,\; 88.6^\circ,\; 89.7^\circ$ \\
$L_2 \leftrightarrow L_3$ & $87.8^\circ,\; 88.2^\circ,\; 89.2^\circ$ \\
$L_3 \leftrightarrow L_4$ & $88.3^\circ,\; 88.8^\circ,\; 89.7^\circ$ \\
$L_4 \leftrightarrow L_5$ & $88.7^\circ,\; 89.6^\circ,\; 89.7^\circ$ \\
\bottomrule
\end{tabular}
\caption{Principal angles between adjacent layers' top-3 singular subspaces.}
\label{tab:principal_angles}
\end{table}

\paragraph{Experiment 2: Per-layer behavior.}
We characterize each coupling layer by two complementary quantities: the relative point-wise displacement $\|f(y)-y\|/\|y\|$, which measures whether the layer performs a non-trivial transformation rather than collapsing to the identity, and the per-layer amplification factor (the empirical Lipschitz factor probed in random directions), which measures how much it can stretch perturbations. Both are computed for every coupling layer of conditional flows trained on all 11 datasets (Table~\ref{tab:displacement}). The mean displacement of $0.390$ shows that layers do meaningful work, while the geometric-mean amplification of $1.049$ confirms that the per-layer Lipschitz factor is close to $1$, satisfying the precondition required for Assumption~\ref{ass:indep-app} to deliver a useful linear-growth bound.

\begin{table}[htbp]
\centering
\begin{tabular}{lc}
\toprule
\textbf{Metric} & \textbf{Value} \\
\midrule
Mean displacement per layer & 0.390 \\
Median & 0.383 \\
Max & 0.586 \\
Per-layer amplification (geo.\ mean) & 1.049 \\
\bottomrule
\end{tabular}
\caption{Per-layer relative displacement $\|f(y)-y\|/\|y\|$ across the conditional flows.}
\label{tab:displacement}
\end{table}

\paragraph{Experiment 3: Perturbation amplification.}
We test the consequence directly. For each conditional flow we add a random-direction perturbation of norm $\varepsilon{=}0.01$ to the input and measure the resulting relative output change $\|f(y+\varepsilon u)-f(y)\|/\varepsilon$ after the full sequence of 18 atomic transforms (Table~\ref{tab:amplification}). Random-direction probing does not certify a worst-case Lipschitz bound, but, combined with the near-orthogonality of dominant singular directions verified in Experiment~1, it provides a faithful empirical estimate of the typical amplification a perturbation experiences in the trained flow. The mean probed amplification of $2.38\times$ closely matches the per-layer factor of $1.049$ compounded over $18$ layers ($1.049^{18}\approx2.38$), and is many orders of magnitude smaller than what even mildly larger per-layer factors would produce (Table~\ref{tab:amplification_comparison}).

\begin{table}[htbp]
\centering
\begin{tabular}{lc}
\toprule
\textbf{Metric} & \textbf{Value} \\
\midrule
Mean total amplification & $2.38\times$ \\
Std & 0.22 \\
Min / Max & 1.94 / 2.79 \\
Per-layer geometric mean & 1.049 \\
\bottomrule
\end{tabular}
\caption{Perturbation amplification across the 18 atomic transforms. The observed $O(1)$ amplification is consistent with linear Lipschitz growth.}
\label{tab:amplification}
\end{table}

\begin{table}[htbp]
\centering
\begin{tabular}{cc}
\toprule
\textbf{Per-layer $\sigma$} & \textbf{Total ($\sigma^{18}$)} \\
\midrule
\textbf{1.049 (ours)} & \textbf{$2.38\times$} \\
1.1 & $5.6\times$ \\
1.5 & $1{,}478\times$ \\
2.0 & $262{,}144\times$ \\
\bottomrule
\end{tabular}
\caption{Comparison of observed amplification with hypothetical exponential growth at different per-layer rates.}
\label{tab:amplification_comparison}
\end{table}

The three experiments together pin down distinct links of the same chain: non-alignment between layer Jacobians (Experiment~1), near-unit per-layer amplification (Experiment~2), and bounded global amplification of the full flow (Experiment~3). Their numerical agreement places the trained flows in the regime where the linear Lipschitz bound underlying Theorem~\ref{thm:gen-bound} is empirically realized.

\subsection{Top-3 Model Identification}
\label{app:top3}

For each dataset we form two unordered sets of size three: \textsc{GT-Top3}, the three highest-scoring models under the supervised metric, and \textsc{FLARE-Top3}, the three highest-scoring models under \methodname's information-sufficiency score $I_s$. The reported overlap is the number of models shared by the two sets, ranging from $0$ to $3$, and is order-agnostic. Ties in either ranking are broken in a fixed lexicographic order over model names; with the gaps observed in our pool ($N{=}8$) this only affects fewer than three borderline assignments.

The Top-3 overlap targets the practitioner-relevant question of whether the ranker surfaces the genuinely best models for deployment. In practice only the top few candidates are actually deployed, so a low full-ranking Spearman $\rho$ does not by itself indicate a poor selector: as long as the Top-3 is correct, mid- and bottom-rank shuffles are inconsequential. By contrast, Spearman $\rho$ over the full $N{=}8$ list (used in the main results) over-penalises such harmless permutations, while Top-1 is too sensitive to a single GT measurement. Top-3 strikes a compromise: with $N{=}8$ the random-selection baseline yields an expected overlap of only $9/8{=}1.125$ and produces a perfect $3/3$ match less than $2\%$ of the time, so even $2/3$ already strongly suggests a non-trivial signal.

Table~\ref{tab:top3_full} reports the per-dataset breakdown. Across datasets, the four with $3/3$ overlap (FunPang, LIMIT, Aug-STSB, LivNLP-STS) are also those with the largest GT gap between the top-3 cluster and the rest, which makes the decision boundary unambiguous. Conversely, datasets where the supervised top-3 scores are tightly clustered tend to yield smaller overlaps, suggesting that the difficulty of Top-3 identification is largely set by the GT separability of the candidate models.

\begin{table*}[!t]
\centering
\begin{tabular}{llllc}
\toprule
\textbf{Dataset} & \textbf{Task} & \textbf{GT Top-3} & \textbf{FLARE Top-3} & \textbf{Overlap} \\
\midrule
Aug-STSB & STS & Linq, SFR, GritLM & SFR, GritLM, Linq & 3/3 \\
LivNLP-STS & STS & Linq, SFR, BGE & SFR, Linq, BGE & 3/3 \\
LIMIT & Retrieval & Zeta, Linq, SFR & Linq, SFR, Zeta & 3/3 \\
FunPang & Clustering & Linq, Zeta, SFR & Zeta, Linq, SFR & 3/3 \\
\midrule
AIR-Bench & Retrieval & Zeta, SFR, Qwen2 & Zeta, GritLM, SFR & 2/3 \\
arXiv '25 & Retrieval & GritLM, Linq, Zeta & Linq, SFR, Zeta & 2/3 \\
Reasoning & Clustering & SFR, GritLM, Zeta & Zeta, SFR, Linq & 2/3 \\
gtfintechlab & Class. & Zeta, BGE, Linq & Zeta, Linq, SFR & 2/3 \\
apt-eval & Class. & GritLM, Linq, BGE & Linq, BGE, Qwen2 & 2/3 \\
Philo-STS & STS & Linq, Qwen2, SFR & SFR, Linq, Zeta & 2/3 \\
\midrule
BhashaBench & Class. & GritLM, Qwen2, BGE & Zeta, GritLM, SFR & 1/3 \\
\bottomrule
\end{tabular}
\caption{Top-3 model identification across all 11 datasets. \methodname\ achieves exact $3/3$ overlap on $4/11$ (36.4\%) datasets and at least $2/3$ overlap on $10/11$ (90.9\%) datasets.}
\label{tab:top3_full}
\end{table*}

\subsection{Bootstrap Confidence Intervals}
\label{app:bootstrap_ci}

The most actionable risk for a small candidate pool ($N{=}8$) is that the ranking depends on which models happen to be included. To probe this, we run a leave-one-out bootstrap over the 8-model pool on the per-dataset $8\times8$ information-sufficiency matrix: for each of the $8$ source models we drop one model at a time, recompute the per-source IS score and Spearman $\rho$ against the supervised ranking on the remaining $7$ models, and report the range $[\rho_{\min}, \rho_{\max}]$ across the $8$ resampled replicates as a non-parametric confidence interval. The aggregate row in Table~\ref{tab:bootstrap_ci} averages the 11 per-dataset point estimates and the per-dataset $[\rho_{\min}, \rho_{\max}]$ bounds.

The width of the resampled range (Table~\ref{tab:bootstrap_ci}) is itself informative. The narrowest ranges (FunPang $[0.857, 0.964]$, BhashaBench $[0.739, 0.919]$) appear on datasets where the IS scores are well separated, so removing any single model leaves the remaining ordering largely intact. Conversely, the widest range (gtfintechlab $[0.143, 0.607]$) is on a small classification benchmark with tightly clustered ground-truth scores, where one model's inclusion can flip several adjacent ranks. In every dataset $\rho_{\min}$ remains positive, and the aggregate lower bound ($0.583$) stays well above $0$, so the positive-correlation finding does not hinge on the inclusion of any particular model.

\subsection{Weight Perturbation Robustness}
\label{app:perturbation}

For each trained conditional flow we form a perturbed copy by adding i.i.d.\ Gaussian noise to every parameter tensor, scaled to the tensor's own magnitude: for tensor $W$ with mean absolute magnitude $\overline{|W|}$, $\Delta W\sim\mathcal{N}(0,(\sigma\,\overline{|W|})^2 I)$. Per-tensor scaling avoids a single global noise level overwhelming small-magnitude tensors (e.g.\ ActNorm scales) while leaving large-magnitude ones untouched, and gives the standard flat-minimum probe interpretation~\citep{hochreiter1997flat,keskar2016large}. We sweep $\sigma\in\{1\%,2\%,5\%,10\%,20\%\}$ with $3$ noise draws per configuration, evaluate the conditional NLL of the clean and perturbed flows on the held-out validation split, and report the relative change $\Delta\mathrm{NLL}/\mathrm{NLL}$ in Table~\ref{tab:perturbation}.

The median relative NLL change stays below $0.03\%$ for $\sigma\le 5\%$ and reaches only $1.22\%$ at $\sigma{=}20\%$, indicating that trained flows sit in flat basins and are essentially insensitive to small parameter perturbations. The widening gap between median and mean shows that sensitivity is concentrated in a small number of pairs rather than being a generic property of the ensemble.

\begin{table}[!t]
\centering
\begin{tabular}{lcccc}
\toprule
\textbf{$\sigma$} & \textbf{Median} & \textbf{Mean} & \textbf{Std} & \textbf{Max} \\
\midrule
1\% & +0.00\% & +0.09\% & 0.54\% & 5.11\% \\
2\% & +0.01\% & +0.29\% & 1.13\% & 9.58\% \\
5\% & +0.02\% & +1.21\% & 3.52\% & 38.3\% \\
10\% & +0.13\% & +3.89\% & 8.97\% & 81.3\% \\
20\% & +1.22\% & +16.4\% & 36.6\% & 389\% \\
\bottomrule
\end{tabular}
\caption{Relative NLL change (\%) under weight perturbation across 616 model pairs from 11 datasets. Median values are reported alongside mean, std, and max. At $\sigma{\le}5\%$, the median NLL change remains below $0.03\%$, indicating that the vast majority of trained flows reside in flat, stable minima.}
\label{tab:perturbation}
\end{table}

\end{document}